\documentclass[conference]{IEEEtran}  
\usepackage{cite}
\usepackage{textcomp}  
\usepackage{multicol}
\usepackage{bussproofs}
\usepackage{color} 
\usepackage[table]{xcolor}
\usepackage{times}
\usepackage{tikz}   
\usepackage[framemethod=TikZ]{mdframed}
\usepackage{marvosym} 
\usepackage{listings}
\usepackage{caption} 
\usepackage[sans]{dsfont}
\usepackage{verbatim}
\usepackage{amssymb} 
\usepackage[lofdepth,lotdepth]{subfig}
\usepackage{amsfonts}
\usepackage{amsmath}

\newcommand{\rewrite}{\code{REWRITE}}  
\newcommand{\negMax}{\ensuremath{\code{neg}\_\code{max}}}
\newcommand{\recurseReduce}{\code{recursiveReduce}}
  
\newcommand{\intFrame}{\ensuremath{\mathfrak{M}}}
 
\newcommand{\squash}{\code{EXTRACT}}  
  
\newcommand{\pcomplete}{\code{PCOMPLETE}}

\newcommand{\exptime}{\code{EXPTIME}}

\usepackage{multirow} 
\usepackage{wrapfig}  

\usepackage{cite}

\newcommand{\sat}{\code{UNSAT}}

\newcommand{\formulasize}{\ensuremath{\code{f}\_\code{size}}}

\newcommand{\hide}[1]{}  
\newcommand*\circled[1]{\tikz[baseline=(char.base)]{
\node[shape=circle,draw,inner sep=1pt] (char) {\ensuremath{#1}};}}

\definecolor{lightsteelblue}{RGB}{176,196,222}

\newcommand{\code}[1]{{\ensuremath{\tt #1}}}

\newcommand{\normalised}{\code{normal}}

\newtheorem{postulate}{Postulate}

\newtheorem{lemma}{Lemma}
\newtheorem{theorem}{Theorem}
\newtheorem{proposition}{Proposition}
\newtheorem{definition}{Definition}
\newtheorem{corollary}{Corollary} 
\newcommand{\paper}[1]{ }
     
\begin{document}
\title{Gradual Classical Logic for Attributed Objects}  
\author{\IEEEauthorblockN{Ryuta Arisaka}
\IEEEauthorblockA{INRIA Saclay--\^Ile-de-France, 
Campus de l'\'Ecole Polytechnique}}  
\maketitle  
\begin{abstract}     
     `There is knowledge. There is belief. 
        And there is tacit agreement.'
    `We may talk about objects. We may talk about 
    attributes of the objects. Or 
        we may talk both about objects and their attributes.' 
   This work inspects tacit agreements on assumptions about 
   the relation between objects and their attributes, 
   and studies
   a way of expressing them, presenting as the result what 
   we term gradual logic 
   in which the sense of truth gradually shifts. It extends 
   classical logic instances with 
   a new logical connective capturing the object-attribute 
   relation. 
   A formal semantics is presented. Decidability is proved. 
   Para-consistent/epistemic/conditional/intensional/description/combined logics are compared. 
\end{abstract}  
\section{Introduction}    
{\it A short description} There is a book. It is on desk. 
It is titled `Meditations and 
Other Metaphysical Writings'. It, or the document 
from which the English translation was borne, is written by 
Ren{\'e} Descartes. {\it Period} \\ 
\indent I\footnote{`We' is preferred throughout this document save 
where the use of the term is most unnatural.}
have just described a book, not some freely arbitrary book  
but one with  
a few pieces of information: that it is on desk, 
that it has the said title, and that it is authored by 
Descartes. \hide{
It happens very frequently in our communication that  
a general term is made more specific and less ambiguous through addition 
of adjectives. 
}   
Let us suppose that I am with a friend of mine. 
If I simply said {\it There is a book} irrespective of being 
fully conscious 
that the book that I have spotted
is the described book and none others, the 
friend of mine, who is here supposed oblivious of any articles on the desk, would have no reason 
to go against imagining 
whatever that is considered a book, say 
`Les Mis{\'e}rables'. The short statement by itself 
does not forestall such a possibility. 
By contrast, 
if, as in the description provided at the beginning, 
I ask him to think of a laid-on-desk Ren{\'e} Descartes  
book titled `Meditations and Other Metaphysical Writings', 
then there would be certain logical dissonance if 
he should still think of `Les Mis{\'e}rables' as a possible
option that conforms to the given description. In innumerable occasions 
like this example, adjectives (or adverbs or whatever terms that 
fulfil the same purpose) are 
utilised to 
disambiguate terms that may denote more than what we intend to communicate.\\
\indent This feature of natural languages, allowing  
formulation 
of a precise enough concept through coordination of 
(1) broad concepts and (2) attributes that narrow down their possibilities, 
is 
a very economical and suitable one for us. For 
imagine otherwise that every word exactly identifies 
a unique and indivisible object around us, then 
we would have no abstract concepts such as 
generalisation or composition 
since generalisation must assume specificity and 
composition decomposability of what result from the process,
neither of which accords with the proposed code of the alternative 
language. While it is certain that 
concepts expressible 
in the alternative language sustain 
no degree of ambiguity in what they refer to, and in this sense  
it may be said to have an advantage to our languages,  
the absence of abstract concepts
that we so often rely upon for 
reasoning is rather grave a backlash that 
would stem its prospect for wide circulation, because - 
after all - who is capable of showing knowledge of 
an object that (s)he has never seen before; then
who could confidently assert that his/her 
listener could understand any part of his/her  
speech on a matter that only he/she knows of if all of us 
were to adopt the alternative language?   
By contrast, concepts in our languages, being 
an identifier of a group rather than an individual, allow 
generation of a vast domain of discourse with a relatively small
number of them in aggregation, \emph{e.g.} `book' and `title' cover anything 
that can be understood as a book and/or a title, and they at the same 
time enable refinement, 
\emph{e.g.} `title'd `book' 
denotes only those books that are titled. The availability 
of mutually influencing generic concepts adds to so much 
flexibility in our languages. \\
\hide{
\indent Many logics currently available 
to us, however, do not primitively capture the essential feature, 
the generality/specificity relation, 
of our languages; or they could but 
only at the expense of explicit predicates' addition to their 
propositional fragments.\\  
}
\indent In this document, we will be interested in  
primitively
representing the particular relation between objects/concepts (no 
special distinction between the two hereafter) and 
what may form 
their attributes, which will lead to development of a new logic. Our domain of discourse will range 
over certain set of (attributed) objects (which may themselves 
be an attribute to other (attributed) objects) and pure attributes that 
presuppose existence of some (attributed) object as their host.  
Needless to say, when we talk about or even just imagine
an object with some explicated attribute, the attribute 
must be found among all that can become an 
attribute to it. To this extent it is confined within 
the presumed existence of the object. 
The new logic intends to address certain phenomena around attributed 
objects which 
I think are reasonably common to us but which may not be 
reasonably expressible in classical logic. Let us turn 
to an example for illustration of the peculiar behaviour 
that attributed objects often present to us.  
\subsection{On peculiarity of attributed objects 
as observed in negation, and on the truth `of' classical logic}  
{\it Episode}
Imagine that there is 
a tiny hat shop in our town, having the following in stock: 
\begin{enumerate}
  \item 3 types of hats: orange hats, green hats 
    ornamented with 
    some brooch, and blue hats decorated with 
    some white hat-shaped accessory, of which only the green 
    and the blue hats are displayed in the shop. 
  \item 2 types of shirts: yellow and blue, 
    of which only the blue shirts are displayed in the shop. 
\end{enumerate}
Imagine also that a young man has come to the hat shop.  
After a while he asks the shop owner, a lady of many a year of experience 
in hat-making; 
``Have you got a yellow hat?'' Well, obviously there are 
no yellow hats to be found in her shop. She answers; 
``No, I do not have it in stock,'' negating 
the possibility that there is one in stock at her shop at the present
point of time.  {\it Period} \\
\indent But ``what is she actually denying about?'' is the inquiry that I 
consider 
pertinent to this writing. We ponder; in delivering the answer, 
the question posed may have allowed her to infer that the young man was 
looking for 
a hat, a yellow hat in particular. Then the answer
may be followed by she saying; ``\dots but I do have hats with 
different colours including ones not currently displayed.'' 
That is, while she denies the presence of a yellow hat, 
she still presumes the availability of hats of which she 
reckons he would like to learn. 
It does not appear so unrealistic, in fact, to suppose such a thought 
of hers that he may be ready to 
compromise his preference for a yellow hat with 
some non-yellow one, 
possibly an orange one in stock, given its comparative closeness 
in hue to yellow. \\
\indent Now, what if the young man turned out to be a town-famous  
collector of yellow articles? Then it may be that 
from his question she had divined instead 
that he was looking for something yellow, 
a yellow hat in particular, in which case her answer 
could have been a contracted form of ``No, I do not have 
it in stock, but I do have a yellow shirt nonetheless (as you 
are looking after, I suppose?)'' \\
\indent Either way, these somewhat-appearing-to-be partial 
negations contrast with 
classical negation with which her 
answer can be interpreted only as that she does not have a yellow 
hat, nothing less, nothing more, with no restriction in the range 
of possibilities outside it. \\ 
\indent An analysis that I attempt regarding this 
sort of usual every-day phenomenon around concepts and 
their attributes, which leads for example to a case where 
negation of some concept with attributes 
does not perforce entail negation of the concept itself but only that 
of the attributes, 
is that presupposition of a concept 
often becomes too strong 
in our mind to be invalidated. Let us proceed in 
allusion to logical/computer science terminologies. In classical 
reasoning that we are familiar with, 
1 - truth - is what we should consider is our truth 
and 0 - falsehood - is what we again should consider is our non-truth.   
When we suppose a set of true atomic propositions 
{\small $p, q, r, \cdots$} under some possible interpretation of them, 
the truth embodied in 
them does - by definition - neither transcend the truth  
 that the 1 signifies nor go below it. The innumerable 
true propositions miraculously sit on the given definition 
of what is true, 1. By applying alternative interpretations, we may
have a different set of innumerable true propositions possibly 
differing from the {\small $p, q, r, \cdots$}. However, no  
interpretations
are meant to modify the perceived significance of the truth  
which   
remains immune to them. Here what renders the truth 
so immutable is the assumption of classical logic that 
no propositions that cannot be given a truth value by means of 
the laws 
of classical logic may appear as a proposition: there is nothing that is 30 \% true, and also nothing that 
is true by the probability of 30 \% unless, of course, 
the probability of 
30 \% should mean to ascribe to our own confidence level, which I here 
assume is 
not part of the logic, of the proposition 
being true. \\
\indent However, one curious fact is that the observation made so far 
can by no means 
preclude a 
deduction that, {\it therefore} and no 
matter how controversial it may appear, 
the meaning of the truth, so long as it can be observed only through 
the interpretations that force the value of propositions to go 
coincident with it and only through examination on the nature\footnote{Philosophical, that is, real, reading 
of the symbols {\small $p, q, r,\dots$}.} 
 of 
those propositions that were 
made true by them, must be invariably dependant on 
the delimiter of our domain of discourse, the 
set of propositions; 
on the presupposition of which are sensibly meaningful the 
interpretations; 
on the presupposition of which, 
in turn, is possible classical logic.
Hence, quite despite 
the actuality that for any set of propositions as can form 
a domain of discourse for classical logic it is sufficient that 
there 
be only one truth, it is not {\it a priori} possible 
that we find by certainty any relation to hold between such 
individual truths and the universal truth, if any, whom we cannot 
hope to successfully 
invalidate. Nor is it {\it a priori} possible to sensibly impose a 
restriction on 
any domain of discourse for classical reasoning to one that is 
consistent with the universal truth, provided again that 
such should exist. But, then, it is not by the force of necessity 
that, having a pair of domains of discourse, we find one 
individual truth and the other wholly interchangeable.  
In tenor, suppose that 
truths are akin to existences, then just as 
there are many existences, so are many truths, every one of which 
can be subjected to classical reasoning, but no distinct pairs of which
{\it a priori} exhibit a trans-territorial compatibility.  
But the lack of compatibility also gives rise 
to a possibility of dependency among them within  
a meta-classical-reasoning 
that recognises the many individual truths at once. In situations 
where some concepts in a domain of discourse over which reigns 
a sense of truth become too strong an assumption to be 
feasibly falsified, the existence of the concepts becomes 
non-falsifiable during the discourse of existences of 
their attributes (which form another domain of discourse); it becomes a delimiter of classical reasoning, that is, it becomes a `truth' for them. 
\subsection{Gradual classical logic: a logic for attributed objects}
It goes hopefully without saying that 
what I wished 
to impart through the above fictitious episode 
was not so much about which negation 
should take a precedence over the others as about   
the distinction of objects and what may form their attributes, 
\emph{i.e.} about  
the inclusion relation to hold between the two and about how 
it could restrict domains of discourse. If we are 
to assume attributed objects as primitive entities 
in a logic, we for example do not just have the negation 
that negates the presence of an attributed object (attributed-object 
negation); on the other hand, 
the logic should be able to express the negation that 
applies to an attribute only (attribute negation) and, complementary, 
we may also consider the negation 
that applies to an object only (object negation). We should also consider 
what it may mean to conjunctively/disjunctively have several 
attributed objects and should attempt a construction of the logic 
according to the analysis. I call the logic derived from 
all these analysis {\it gradual classical logic} in which 
the `truth', a very fundamental property of classical 
logic, gradually shifts
by domains of discourse moving deeper into 
attributes of (attributed) objects. For a special emphasis, here the gradation in truth occurs only 
in the sense that 
is spelled out in 
the previous sub-section. 
One in particular should not confuse this logic with multi-valued logics 
\cite{Gottwald09,Hajek10} 
that have multiple truth values in the same domain of discourse, for 
any (attributed) object in gradual classical logic assumes 
only one out of the two usual possibilities: either it is true 
(that is, because we shall essentially consider conceptual 
existences, it is synonymous to saying that it exists) or it is false (it 
does not exist). In this sense it is indeed classical logic. But 
in some sense - because we can observe transitions in the sense of 
the `truth' within the logic 
itself - it has a bearing of meta-classical logic.  
As for 
inconsistency, if there is an inconsistent argument within  
a discourse on attributed objects, wherever it may be that 
it is occurring, 
the reasoning part of which is inconsistent cannot be 
said to be consistent. For this reason it remains in gradual classical logic 
just as strong as is in standard classical logic.  
\subsection{Structure of this work}    
Shown below is the organisation of this work. 
\begin{itemize}
  \item Development of gradual classical logic 
    (Sections {\uppercase\expandafter{\romannumeral 1}} 
    and {\uppercase\expandafter{\romannumeral 2}}).  
  \item A formal semantics 
    of gradual classical logic and a proof that it is 
    not para-consistent/inconsistent 
    (Section {\uppercase\expandafter{\romannumeral 3}}).  
  \item Decidability of gradual classical logic 
    (Section {\uppercase\expandafter{\romannumeral 4}}). 
  \item Conclusion and discussion on related thoughts:  para-consistent logics, epistemic/conditional 
logics, intensional/description logics, and combined logic
    (Section {\uppercase\expandafter{\romannumeral 5}}). 
\end{itemize}
\section{Gradual Classical Logic: Logical Particulars} 
In this section we shall look into logical particulars of 
gradual classical logic. Some familiarity 
with 
propositional classical logic, in particular with 
how the logical connectives behave, is presumed. 
Mathematical transcriptions of 
gradual classical logic are found in the next section. 
\subsection{Logical connective for object/attribute and 
interactions with negation ({\small $\gtrdot$} and {\small $\neg$})}
It was already mentioned that the inclusion relation 
that is implicit when we talk about an attributed object 
shall be primitive in the proposed gradual classical logic. 
We shall dedicate the symbol {\small $\gtrdot$} to represent it.
The usage 
of the new connective is fixed to take either of the forms
{\small $\code{Object}_1 \gtrdot \code{Object}_2$}   
or {\small $\code{Object}_1 \gtrdot \code{Attribute}_2$}. 
Both denote an attributed object. In the first case, 
{\small $\code{Object}_1$} is a more generic 
object than {\small $\code{Object}_1 \gtrdot \code{Object}_2$} 
({\small $\code{Object}_2$} acting as an attribute to 
{\small $\code{Object}_1$}  
makes {\small $\code{Object}_1$} more specific).  
In the second case, we have a pure attribute which is not itself an object.
Either way a schematic reading is as follows: 
``It is true that {\small $\code{Object}_1$} is, 
and it is true that {\small $\code{Object}_1$}  
has an attribute of {\small $\code{Object}_2$} (, or 
of {\small $\code{Attribute}_2$}).'' Given 
an attributed object {\small $\code{Object}_1 \gtrdot 
\code{Ojbect}_2$} (or {\small $\code{Object}_1 \gtrdot 
\code{Attribute}_2$}),  
{\small $\neg (\code{Object}_1 \gtrdot \code{Object}_2)$} 
expresses its attributed object negation, 
{\small $\neg \code{Object}_1 \gtrdot \code{Object}_2$} 
its object negation and 
{\small $\code{Object}_1 \gtrdot \neg \code{Object}_2$} 
its attribute negation. Again schematic readings for them 
are, respectively;   
\begin{itemize}
  \item
It is false that the attributed object 
{\small $\code{Object}_1 \gtrdot \code{Object}_2$} is 
(\emph{Cf}. above for the reading of `an attribute object is'). 
\item 
It is false that {\small $\code{Object}_1$} is, but 
it is true that some non-{\small $\code{Object}_1$} is  
which has an attribute of {\small $\code{Object}_2$}. 
\item 
It is true that {\small $\code{Object}_1$} is, 
but it is false that it has an attribute 
of {\small $\code{Object}_2$}.
\end{itemize}
The presence 
of negation flips ``It is true that \ldots'' into 
``It is false that \ldots'' and vice versa. But it should be also 
noted how negation acts in attribute negations and 
object/attribute negations. Several specific examples\footnote{I do not 
pass judgement on what is reasonable and what is not here, 
as my purpose is to illustrate the reading of 
{\small $\gtrdot$}. So there 
are ones that ordinarily appear to be not very reasonable.}
constructed parodically from the items in the hat shop episode are;  
\begin{enumerate}
  \item {\small $\text{Hat} \gtrdot \text{Yellow}$}:        
    It is true that hat is, and it is true that it is yellow(ed). 
\item {\small $\text{Yellow} \gtrdot \text{Hat}$}:     
  It is true that yellow is, and it is true that it is hatted.    
\item {\small $\text{Hat} \gtrdot \neg \text{Yellow}$}:       
    It is true that hat is, but it is false that it is yellow(ed).\footnote{In the rest, this -ed to 
    indicate an adjective is assumed clear and is omitted another 
    emphasis.}      
\item {\small $\neg \text{Hat} \gtrdot \text{Yellow}$}:  
    It is false that hat is, but it is true that yellow object (which 
    is not hat) is.  
      \item {\small $\neg (\text{Hat} \gtrdot \text{Yellow})$}:          
  Either it is false that hat is, or if it is true that hat is,
  then it is false that it is 
  yellow.    
\end{enumerate}
\subsection{Object/attribute relation and 
conjunction ({\small $\gtrdot$} and {\small $\wedge$})}  
We examine specific examples first involving 
{\small $\gtrdot$} and {\small $\wedge$} (conjunction), 
and then observe what the readings imply.  
\begin{enumerate}
  \item {\small $\text{Hat} \gtrdot \text{Green} \wedge 
    \text{Brooch}$}: It is true that 
    hat is, and it is true that it is green and brooched. 
  \item {\small $(\text{Hat} \gtrdot \text{Green}) 
    \wedge (\text{Hat} \gtrdot \text{Brooch})$}: 
    for one, it is true that 
    hat is, and it is true that it is 
    green; for one, it is true that 
    hat is, and it is true that it is 
    brooched.  
  \item {\small $(\text{Hat} \wedge \text{Shirt}) 
    \gtrdot \text{Yellow}$}: 
    It is true that 
    hat and shirt are, and it is true that 
    they are yellow.  
  \item {\small $(\text{Hat} \gtrdot \text{Yellow}) 
    \wedge (\text{Shirt} \gtrdot \text{Yellow})$}:  
    for one, it is true that 
    hat is, and it is true that 
    it is yellow; for one, it is true that 
    shirt is, and it is true that 
    it is yellow.  
\end{enumerate} 
By now it has hopefully become clear that by 
{\it existential facts as truths} I do not mean how many of a given 
(attributed) object exist: in 
gradual classical logic, cardinality of 
objects, which is an important pillar in the philosophy of linear logic 
\cite{DBLP:journals/tcs/Girard87} and that of its kinds of 
so-called resource logics, 
is not what it must be responsible for, but 
only the facts themselves of whether any of them 
exist in a given domain of discourse, which 
is in line with classical logic.\footnote{
That proposition A is true and that proposition A is true 
mean that proposition A is true; the subject of this sentence 
is equivalent to the object of its.}
Hence they univocally assume a singular 
than plural form, as in the examples inscribed so far. That 
the first and the second, and the third and the fourth, 
equate is then a trite observation. Nevertheless, 
it is still important that we analyse them with a sufficient 
precision. In the third and the fourth where 
the same attribute is shared among several objects, the attribute 
of being yellow ascribes to all of them. 
Therefore those expressions are a true statement only if  
(1) there is an existential fact that both hat and shirt are  
and (2) being yellow is true for the existential fact (formed 
by existence of hat and that of shirt). 
Another example is found in Figure \ref{first_figure}.  
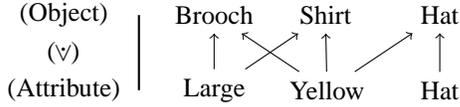
\begin{figure}[!h]
  \begin{center}  
    \begin{tikzpicture}     
      \draw[thick] (-2, 0) -- (-2, -1);   
      \node (a) at (-1, 0) {$\text{Brooch}$}; 
      \node (b) at (0.48, 0) {$\text{Shirt}$}; 
      \node (c) at (2, 0) {$\text{Hat}$};  
      \node (d) at (-1, -1) {$\text{Large}$}; 
      \node (e) at (0.5, -1) {$\text{Yellow}$};  
      \node (f) at (2, -1) {$\text{Hat}$};  
      \draw [->] (d) -- (a); 
      \draw [->] (d) -- (b); 
      \draw [->] (e) -- (a); 
      \draw [->] (e) -- (b); 
      \draw [->] (e) -- (c); 
      \draw [->] (f) -- (c); 
      \node (g) at (-3, 0) {$\text{(Object)}$};   
      \node (h) at (-3, -1) {$\text{(Attribute)}$};  
      \node (i) at (-3, -0.5) {(\reflectbox{\rotatebox[origin=c]{-90}{$\gtrdot$}})};  
    \end{tikzpicture} 
  \end{center}
  \caption{Illustration of an expression 
  {\small $((\text{Brooch} \wedge 
  \text{Shirt}) \gtrdot  
  \text{Large}) \wedge 
  ((\text{Brooch} \wedge 
  \text{Shirt} \wedge \text{Hat}) \gtrdot 
  \text{Yellow}) \wedge 
  (\text{Hat} \gtrdot \text{Hat})$}: the existential 
  fact of the attribute large 
  depends on the existential facts 
  of brooch and shirt; the existential 
  fact of the attribute of being yellow 
  depends on the existential facts 
  of brooch, shirt and hat; and 
  the existential fact of the attribute 
  hat depends on the existential fact of 
  hat to which it is an attribute.}
  \label{first_figure}
\end{figure}
  
\subsection{Object/attribute relation and disjunction 
({\small $\gtrdot$} and {\small $\vee$})} 
We look at examples first.   
\begin{enumerate}
  \item {\small $\text{Hat} \gtrdot (\text{Hat} \vee 
    \text{Brooch})$}: 
    It is true that hat is, and it is true that 
    it is either hatted or brooched.  
  \item {\small $(\text{Hat} \gtrdot \text{Hat}) 
    \vee (\text{Hat} \gtrdot \text{Brooch})$}: 
    At least either that it is true that hat is and 
    it is true that it is hatted, or that
    it is true that hat is and it is true that 
    it is brooched.  
  \item {\small $(\text{Hat} \vee \text{Shirt}) \gtrdot 
    \text{Yellow}$}: It is true that 
    at least either hat or shirt is, and 
    it is true that whichever is existing (or both) is (or are) yellow. 
  \item {\small $(\text{Hat} \gtrdot \text{Yellow}) \vee 
    (\text{Shirt} \gtrdot \text{Yellow})$}:  
    At least either it is true that hat is and it is 
    true that it is yellow, or 
    it is true that shirt is and it is true that 
    it is yellow.  
\end{enumerate}
Just as in the previous sub-section, here again 1) and 2), 
and 3) and 4) are equivalent. 
However, in the cases of 3) and 4) here, we 
have that the existential fact of the attribute yellow 
depends on that of hat or shirt, whichever is existing, or 
that of both if they both exist.\footnote{In classical logic, that proposition 
A or proposition B is true means 
that at least one of the proposition A or the proposition B is true though 
both can be true. 
Same goes here.}
\subsection{Nestings of object/attribute relations}
An expression of the kind {\small $(\code{Object}_1 \gtrdot 
\code{Object}_2) \gtrdot \code{Object}_3$} 
is ambiguous. But we begin by listing examples 
and then move onto analysis of the readings of the nesting of 
the relations. 
\begin{enumerate}
\item  {\small $(\text{Hat} \gtrdot \text{Brooch}) \gtrdot 
    \text{Green}$}: It is true that hat is, and it is true that 
    it is brooched. It is true that the object thus 
    described is green.  
  \item {\small $\text{Hat} \gtrdot (\text{Hat} \gtrdot \text{White})$}:
    It is true that hat is, and it is true that 
    it has the attribute of which 
    it is true that hat is and that it is 
    white. (More simply, it is true that hat is, 
    and it is true that it is white-hatted.)
  \item {\small $\neg (\text{Hat} \gtrdot \text{Yellow}) 
    \gtrdot \text{Brooch}$}:   
    Either it is false that hat is, or else it is true that 
    hat is but it is false that it is yellow.\footnote{
    This is the reading of 
    {\small $\neg (\text{Hat} \gtrdot \text{Yellow})$}.}
    If it is false that hat is, then it is true that 
    brooched object (which obviously cannot be hat) is. 
    If it is true that hat is but it is false that it is yellow,  
    then it is true that the object thus described is brooched. 
\end{enumerate}     
Note that to say that Hat {\small $\gtrdot$} Brooch (brooched hat) 
is being green, we must mean to say that the object to 
the attribute of being green, \emph{i.e.} hat, is green. It 
is on the other hand 
unclear if green brooched hat should or should not 
mean that the brooch, an accessory to hat, is also green. 
But common sense about adjectives dictates that 
such be simply indeterminate. It is reasonable
for (Hat {\small $\gtrdot$} Brooch) {\small $\gtrdot$} 
Green, while if we have 
(Hat {\small $\gtrdot$} Large) {\small $\gtrdot$} 
Green, ordinarily speaking it cannot be the case that the attribute of being 
large is green. Therefore 
we enforce that {\small $(\code{Object}_1 \gtrdot \code{Object}_2) 
\gtrdot \code{Object}_3$} amounts to 
{\small $(\code{Object}_1 \gtrdot \code{Object}_3) 
\wedge ((\code{Object}_1 \gtrdot \code{Object}_2)
\vee 
(\code{Object}_1 \gtrdot (\code{Object}_2 \gtrdot \code{Object}_3)))$} in which disjunction as usual captures the 
indeterminacy. 2) poses no ambiguity. 3) is understood in the same 
way as 1). 
\subsection{Two nullary logical connectives 
{\small $\top$} and {\small $\bot$}}
Now we examine the nullary logical connectives  
{\small $\top$} and {\small $\bot$} which 
denote, in classical logic, the concept of the truth and that of 
the inconsistency. In gradual classical logic 
{\small $\top$} denotes the concept of the presence 
and {\small $\bot$} denotes that of the absence. 
Several examples for
the readings are;
\begin{enumerate}
  \item {\small $\top \gtrdot \text{Yellow}$}:   
    It is true that yellow object is. 
  \item {\small $\text{Hat} \gtrdot (\top \gtrdot \text{Yellow})$}:   
    It is true that hat is, and it is true that it has the 
    following attribute of which it is true that 
    it is yellow object. 
  \item {\small $\bot \gtrdot \text{Yellow}$}:   
    It is true that nothingness is, and it is true that 
    it is yellow. 
  \item {\small $\text{Hat} \gtrdot \top$}:  
    It is true that hat is. 
  \item {\small $\text{Hat} \gtrdot \bot$}:    
    It is true that hat is, and it is true that it has no attributes. 
  \item {\small $\bot \gtrdot \bot$}: It is  
    true that nothingness is, and it is true that it has no attributes. 
\end{enumerate}      
1) and 2) illustrate how the sense of the `truth'
is constrained by the object to which it acts as 
an attribute. For the rest, however, 
there is a point around 
the absence which is not so vacuous as not to merit 
a consideration, and to which I in fact append the following postulate. 
\begin{postulate} 
  That which cannot have any attribute is not. 
  Conversely,  
  anything that remains once all the attributes 
  have been removed from a given object is nothingness
  for which 
  any scenario where it comes with an attribute is inconceivable.  
  \label{axiomZero}
\end{postulate}   
With it, 3) which asserts the existence of 
nothingness is contradictory. 4) then behaves as expected in 
that Hat which is asserted with the presence of attribute(s)  
is just as generic a term as Hat itself is. 
5) which asserts the existence of an object with 
no attributes again contradicts Postulate \ref{axiomZero}.  
6) illustrates that any attributed object in some part of which 
has turned out to be contradictory remains contradictory no matter 
how it is to be extended: a {\small $\bot$} 
cannot negate another {\small $\bot$}. \\
\indent But how plausible is the postulate itself? 
Let us imagine hat. If the word evoked in our mind any 
specific hat with specific 
colour and shape, we 
first remove the colour out of it. If the process should make 
it transparent, we then remove the transparentness away from it. And if there should be 
still some things that are by some means perceivable as have 
originated from it, then because they are an attribute 
of the hat, we again remove any one of them. If 
{\it the humanly
no longer detectable something  is not nothingness} is not itself 
contradictory, then there must be still some quality originating 
in the hat that makes the something differ from nothingness. But 
the quality must again be an attribute to the hat, which we 
decisively remove away. Therefore, at least intuition
solidifies the validity of Postulate \ref{axiomZero}. A further 
pursuit on this topic may be useful. For now, 
however, we shall draw a direct support from - among others - 
Transcendental Aesthetic in Critique of Pure Reason (English 
translation \cite{Kant08}), and close the scene. 
\subsection{Sub-Conclusion}  
Gradual classical logic was developed in Section  
{\uppercase\expandafter{\romannumeral 1}} and 
Section {\uppercase\expandafter{\romannumeral 2}}. 
The next two sections Section {\uppercase\expandafter{\romannumeral 3}}
and Section {\uppercase\expandafter{\romannumeral 4}} 
study its mathematical aspects. 
\section{Mathematical mappings: syntax and semantics}   
In this section a semantics 
of gradual classical logic is formalised. 
We assume in the rest of this document;
\begin{itemize}
  \item {\small $\mathbb{N}$} denotes the set of natural numbers 
    including 0.   
  \item {\small $\wedge^{\dagger}$} and {\small $\vee^{\dagger}$} 
    are two binary 
    operators on Boolean arithmetic. 
    The following laws hold; 
    {\small $1 \vee^{\dagger} 1 = 1 \vee^{\dagger} 0 = 0 
    \vee^{\dagger} 1 = 1$}, 
    {\small $0 \wedge^{\dagger} 0 
    = 0 \wedge^{\dagger} 1 = 1 \wedge^{\dagger} 0 = 0$}, 
    and {\small $1 \wedge^{\dagger} 1 = 1$}. 
  \item {\small $\wedge^{\dagger}$}, {\small $\vee^{\dagger}$} 
      {\small $\rightarrow^{\dagger}$}, {\small $\neg^{\dagger}$}, {\small $\exists$} 
    and {\small $\forall$}
    are meta-logical connectives: conjunction, disjunction,\footnote{
    These two symbols are overloaded. Save whether 
    truth values or  the ternary values are supplied as arguments, 
    however, the distinction is clear from the context 
    in which they are used. 
    }
    material implication, negation, existential quantification 
    and universal quantification, whose semantics  
    follow those of standard classical logic.    
    We abbreviate {\small $(A \rightarrow^{\dagger} B) 
    \wedge^{\dagger} (B \rightarrow^{\dagger} A)$} 
    by {\small $A \leftrightarrow^{\dagger} B$}.  
  \item Binding strength of logical or meta-logical connectives is,
    in 
    the order of 
    decreasing precedence;\\
    {\small $[\neg]\! \gg \!
    [\wedge \ \ \vee]\! \gg \!
    [\gtrdot] \gg [\forall \ \ \exists]\! \gg \!
    [\neg^{\dagger}]\! \gg\! [\wedge^{\dagger} \ \  \vee^{\dagger}]
    \! \gg\! [\rightarrow^{\dagger}] 
    \! \gg\! [\leftrightarrow^{\dagger}]$}.    
  \item For any binary connectives {\small $?$}, 
    for any {\small $i, j \in \mathbb{N}$} and 
    for {\small $!_0, !_1, \cdots, !_j$} that are some recognisable 
    entities, 
    {\small $?_{i = 0}^j !_i$} is an abbreviation of 
    {\small $(!_0) ? (!_1) ? \cdots ? (!_j)$}.  
  \item 
    For the unary connective {\small $\neg$},   
    {\small $\neg \neg !$} for some recognisable entity 
    {\small $!$} is an abbreviation of 
    {\small $\neg (\neg !)$}. Further, 
    {\small $\neg^k !$} for some {\small $k \in \mathbb{N}$}  
    and some recognisable entity {\small $!$} 
    is an abbreviation of 
    {\small $\underbrace{\neg \cdots \neg}_k
    !$}.  
  \item For the binary connective {\small $\gtrdot$}, 
    {\small $!_0 \gtrdot !_1 \gtrdot !_2$} 
    for some three recognisable entities 
    is an abbreviation of 
    {\small $!_0 \gtrdot (!_1 \gtrdot !_2)$}. 
\end{itemize}
On this preamble we shall begin. 
\subsection{Development of semantics}
The set of literals in gradual classical logic is denoted by {\small $\mathcal{A}$} 
whose elements are referred to by {\small $a$} with or without 
a sub-script. This set has a countably many number 
of literals. 
Given a literal {\small $a \in \mathcal{A}$}, 
its complement is denoted by {\small $a^c$} which is 
in {\small $\mathcal{A}$}. As usual, we have {\small $\forall 
a \in \mathcal{A}.(a^c)^c = a$}. 
The set {\small  $\mathcal{A} \cup \{\top\} \cup \{\bot\}$} where 
{\small $\top$} and {\small $\bot$} are the two nullary logical 
connectives is denoted by {\small $\mathcal{S}$}. Its elements 
are referred to by {\small $s$} with or without 
a sub-script. Given {\small $s \in \mathcal{S}$}, its 
complement is denoted by {\small $s^c$} which 
is in {\small $\mathcal{S}$}. Here we have    
{\small $\top^c = \bot$} and 
{\small $\bot^c = \top$}. 
The set of formulas is denoted by {\small $\mathfrak{F}$} 
whose elements, {\small $F$} with or without 
a sub-/super-script,
are finitely constructed from the following grammar; \\
\indent {\small $F := s \ | \ F \wedge F \ | \ 
F \vee F \ | \ \neg F \ | \  F \gtrdot F$}\\
We now develop semantics. This is done in two parts: we do not outright 
jump to the definition 
of valuation (which we could, but which we simply do not 
choose in anticipation for later proofs). Instead, just as we only need 
consider negation normal form in classical logic because every 
classical logic formula definable 
has a reduction into a normal form, so 
shall we first define rules for formula reductions 
(for any {\small $F_1, F_2, F_3 \in 
\mathfrak{F}$}):  
\begin{itemize}
  \item {\small $\forall s \in \mathcal{S}.\neg s \mapsto s^c$} 
    ({\small $\neg$} reduction 1). 
  \item {\small $\neg (F_1 \wedge F_2) \mapsto
    \neg F_1 \vee \neg F_2$} ({\small $\neg$} reduction 2). 
  \item {\small $\neg (F_1 \vee F_2) \mapsto
    \neg F_1 \wedge \neg F_2$} ({\small $\neg$} reduction 3). 
  \item {\small $\neg (s \gtrdot F_2) \mapsto
    s^c \vee (s \gtrdot \neg F_2)$} ({\small $\neg$} 
    reduction 4).    
  \item {\small $(F_1 \gtrdot F_2) \gtrdot F_3
     \mapsto (F_1 \gtrdot F_3) \wedge  
     ((F_1 \gtrdot F_2) \vee (F_1 \gtrdot F_2 \gtrdot F_3))$} 
     ({\small $\gtrdot$} reduction 1).
   \item {\small $
     (F_1 \wedge F_2) \gtrdot F_3 \mapsto
     (F_1 \gtrdot F_3) \wedge (F_2 \gtrdot F_3)$} 
     ({\small $\gtrdot$} reduction 2). 
   \item {\small $
     (F_1 \vee F_2) \gtrdot F_3 \mapsto 
     (F_1 \gtrdot F_3) \vee (F_2 \gtrdot F_3)$} 
     ({\small $\gtrdot$} reduction 3).  
   \item {\small $F_1 \gtrdot (F_2 \wedge F_3) 
     \mapsto (F_1 \gtrdot F_2) \wedge (F_1 \gtrdot F_3)$} 
     ({\small $\gtrdot$} reduction 4). 
   \item {\small $F_1 \gtrdot (F_2 \vee F_3) 
     \mapsto (F_1 \gtrdot F_2) \vee (F_1 \gtrdot F_3)$} 
     ({\small $\gtrdot$} reduction 5). 
\end{itemize}    
\begin{definition}[Valuation frame]        
  Let {\small $\mathcal{S}^*$} denote the set union of (A) the set 
  of finite sequences of elements 
  of {\small $\mathcal{S}$}\footnote{Simply 
  for a presentation purpose, 
  we use comma such as {\small $s_1^*.s_2^*$} for 
  {\small $s_1^*, s_2^* \in \mathcal{S}^*$} to show that 
  {\small $s_1^*.s_2^*$} is an element of 
  {\small $\mathcal{S}^*$} in which 
  {\small $s_1^*$} is the preceding constituent and {\small $s_2^*$} 
  the following constituent of 
  {\small $s_1^*.s_2^*$}.} 
  and 
  (B) a singleton set {\small $\{\epsilon\}$} denoting 
  an empty sequence.  
    We define a valuation frame as a 2-tuple: 
  {\small $(\mathsf{I}, \mathsf{J})$}, where 
  {\small $\mathsf{I}: \mathcal{S}^* \times \mathcal{S} \rightarrow 
      \{0,1\}$} is what we call  
  local interpretation 
  and 
  {\small $\mathsf{J}: \mathcal{S}^* \backslash 
      \{\epsilon\} \rightarrow \{0,1\}  
       $} is what we call gloal interpretation. 
       The following are defined to satisfy.   
       \begin{description} 
           \item[Regarding local interpretation]{\ }
      \begin{itemize}
          \item {\small $[\mathsf{I}(s_0.\dots.s_{k-1}, \top) = 1]$}\footnote{
        When {\small $k = 0$}, 
        we assume that 
        {\small $[\mathsf{I}(s_0.\dots.s_{k-1}, s_k) 
        = \mathsf{I}(\epsilon, s_0)]$}. Same applies 
    in the rest.} 
          ({\small $\mathsf{I}$} valuation of $\top$).
      \item {\small $[\mathsf{I}(s_0.\dots.s_{k-1}, \bot) = 0]$} 
          (That  of $\bot$).  
      \item {\small $[\mathsf{I}(s_0.\dots.s_{k-1}, a_{k})
              = 0] \vee^{\dagger} 
              [\mathsf{I}(s_0.\dots.s_{k-1}, a_{k})
= 1]$} 
(That of a literal).  
\item {\small $[\mathsf{I}(s_0.\dots.s_{k-1},a_{k}) = 0]
        \leftrightarrow^{\dagger} 
        [\mathsf{I}(s_0.\dots.s_{k-1},a^c_{k}) = 1]$}
    (That of a complement).   
\item {\small $[{\mathsf{I}(s_0.\dots.s_{k-1}, s_{k}) 
            = \mathsf{I}(s'_0.\dots.s'_{k-1}, s_{k})}]$} 
     (Synchronization condition on {\small $\mathsf{I}$} 
     interpretation; this reflects  
     the dependency of the existential fact of an attribute to 
     the existential fact of objects to which 
     it is an attribute).   
 \end{itemize}
 \item[Regarding global interpretation]{\ }  
     \begin{itemize}
\item {\small $[\mathsf{J}(s_0.\dots.s_k) =  
        1] \leftrightarrow^{\dagger}$} {\small $\forall i \in \mathbb{N}.
        \bigwedge_{i=0}^{\dagger k} 
        [\mathsf{I}(s_0.\dots.s_{i-1},s_i) = 1]$}    (Non-contradictory {\small $\mathsf{J}$} 
    valuation). 
\item {\small $[\mathsf{J}(s_0.\dots.s_k) = 
        0] \leftrightarrow^{\dagger}$} 
    {\small $
        \exists i \in \mathbb{N}.[i \le k] \wedge^{\dagger}
        [\mathsf{I}(s_0.\dots.s_{i-1}, s_i) = 0]$}  
    (Contradictory {\small $\mathsf{J}$} valuation).   
  \end{itemize} 
  \end{description}  
  \label{interpretations}
\end{definition}    
Note that global interpretation is completely characterised by local interpretations, as clear from the definition. 
\begin{definition}[Valuation]    
  Suppose a valuation frame {\small $\mathfrak{M} = (\mathsf{I}, \mathsf{J})$}. The following are defined to hold 
  for all  
  {\small $F_1, F_2\in \mathfrak{F}$} and for all 
  {\small $k \in \mathbb{N}$}:    
\begin{itemize}
    \item {\small $[\intFrame \models s_0 \gtrdot s_1 \gtrdot \dots \gtrdot s_k] =  
  \mathsf{J}(s_0.s_1.\dots.s_k)$}.    
\item {\small $[\intFrame \models F_1 \wedge F_2] = 
  [\intFrame \models F_1] \wedge^{\dagger} 
  [\intFrame \models F_2]$}. 
\item {\small $[\intFrame \models F_1 \vee F_2] = 
  [\intFrame \models F_1] 
  \vee^{\dagger} [\intFrame \models F_2]$}. 
       \end{itemize}  
   \label{model}
\end{definition}     
The 
notions of validity and satisfiability are 
as usual. 
\begin{definition}[Validity/Satisfiability]
    A formula {\small $F \in \mathfrak{F}$} is said 
  to be satisfiable in a valuation frame {\small $\mathfrak{M}$} 
  iff 
  {\small $1 = [\intFrame \models F]$}; 
  it is said to be valid iff it is satisfiable 
  for all the valuation frames; 
  it is said to be invalid iff 
  {\small $0 = [\intFrame \models F]$} for 
  some valuation frame {\small $\intFrame$}; 
  it is said to be unsatisfiable 
  iff it is invalid for all the valuation frames.    
  \label{universal_validity}
\end{definition}      
\subsection{Study on the semantics}
We have not yet formally verified
some important points.  
Are there, firstly, 
any formulas {\small $F \in \mathfrak{F}$} that 
do not reduce into some value-assignable formula? 
Secondly, what if both 
{\small $1 = [\intFrame \models F]$} 
and {\small $1 = [\intFrame \models \neg F]$}, or both 
{\small $0 = [\intFrame \models F]$} and 
{\small $0 = [\intFrame \models \neg F]$} for some 
{\small $F \in \mathfrak{F}$} under some 
{\small $\intFrame$}? 
Thirdly, should it happen that {\small $[\mathfrak{M}
\models F] 
= 0 = 1$} for any formula {\small $F$}, given a
valuation frame? \\
\indent If the first should hold, 
the semantics - the reductions and valuations as were presented in 
the previous sub-section - would 
not assign a value (values) to every member of 
{\small $\mathfrak{F}$} even with the reduction rules made available. 
If the second should hold, we could gain 
{\small $1 = [\mathfrak{M} \models F \wedge \neg F]$},
which would relegate this gradual logic to 
a family of para-consistent logics \cite{Marcos05} - quite out of  
keeping with my intention. And the third should never 
hold, clearly. \\
\indent Hence it must be shown that 
these unfavoured situations do not arise. An outline to the completion 
of the proofs is; 
\begin{enumerate}
  \item to establish that every formula has a reduction 
    through {\small $\neg$} reductions and 
    {\small $\gtrdot$} reductions 
    into some formula {\small $F$} for which it holds that 
    {\small $\forall \mathfrak{M}.
        [\mathfrak{M} \models F] \in \{0, 1\}$}, to 
    settle down the first inquiry. 
  \item to prove that
      any formula {\small $F$} 
      to which a value 0/1
      is assignable {\it without the use of 
      the reduction 
      rules} satisfies for every valuation frame (a) 
      that {\small $[\mathfrak{M} \models F] \vee^{\dagger} 
    [\mathfrak{M}\models
    \neg F] = 1$} and 
{\small $[\mathfrak{M} \models F] \wedge^{\dagger} [\mathfrak{M} \models
    \neg F] = 0$}; and (b) either that 
    {\small $0 \not= 1 = [\intFrame \models F]$} 
    or that {\small $1 \not= 0 = [\intFrame \models F]$},
    to settle down the other inquiries partially. 
\item to prove that
      the reduction through 
    {\small $\neg$} reductions and {\small $\gtrdot$} reductions 
    on any formula {\small $F \in \mathfrak{F}$} 
    is normal in that, 
    in whatever order those reduction rules are 
    applied to {\small $F$}, any {\small $F_{\code{reduced}}$} in 
    the set of possible formulas it reduces into satisfies 
    for every valuation frame either 
    that {\small $[\mathfrak{M} \models 
        F_{\code{reduced}}] = 1$}, or 
    that {\small $[\mathfrak{M} \models 
        F_{\code{reduced}}] = 0$}, 
    for all such {\small $F_{\code{reduced}}$}, 
    to conclude. 
\end{enumerate} 
\subsubsection{Every formula is 0/1-assignable} 
\vspace{-0.1mm}
We state several definitions for the first objective of ours. 
\begin{definition}[Chains/Unit chains]{\ }\\ 
  A chain is defined to be any formula 
  {\small $F \in \mathfrak{F}$} such 
  that 
  {\small $F = F_0 \gtrdot F_1 \gtrdot \dots \gtrdot F_{k+1}$} 
  for {\small $k \in \mathbb{N}$}. 
  A unit chain is defined to be a chain 
  for which {\small $F_i \in \mathcal{S}$} for 
  all {\small $0 \le i \le k+1$}. We denote 
  the set of 
  unit chains by {\small $\mathfrak{U}$}.   
  By the head of a chain {\small $F \in 
  \mathfrak{F}$}, we mean some formula {\small $F_a \in 
  \mathfrak{F}$} satisfying
  (1) that {\small $F_a$} is not in the form 
  {\small $F_b \gtrdot F_c$} for some 
  {\small $F_b, F_c \in \mathfrak{F}$} and (2) that 
  {\small $F = F_a \gtrdot F_d$} 
  for some {\small $F_d \in \mathfrak{F}$}.   
  By the tail of a chain {\small $F \in 
  \mathfrak{F}$}, we then mean some formula 
  {\small $F_d \in \mathfrak{F}$} such that  
  {\small $F = F_a \gtrdot F_d$} for 
  some {\small $F_a$} as the head of {\small $F$}. 
\end{definition} 
\begin{definition}[Unit chain expansion]{\ }\\
  Given any {\small $F \in \mathfrak{F}$}, we 
  say that {\small $F$} is expanded in 
  unit chains only if 
  any chain that occurs in {\small $F$} is a unit chain.  
\hide{  Given 
  any formula {\small $F \in \mathfrak{F}$}, 
  we denote by {\small $[F]_p$}  
  of {\small $\mathfrak{F}$} such that 
  it contains only those formulas (A) that result 
  from applying rules of transformations 
  
  \indent {\small $\{F' (\in \mathfrak{F})\ | \ F' \text{ is 
  expanded in primary chains via \textbf{Transformations} in Definition 3.}\}$}. 
  Likewise, we denote by {\small $[F]_u$} the set;\\
  \indent {\small $\{F' (\in \mathfrak{F})\ | \ F' \text{ is 
  expanded in unit chains via Definition 3.}\}$}.\footnote{
  Given a formula {\small $F$}, it may be that {\small $[F]_u$} 
  is always a singleton set, 
  {\small $F$} always leading to a unique unit expansion 
  via Definition 3. But I do not verify this in this paper, as 
  no later results strictly the stronger statement.}
  }
\end{definition}     
\begin{definition}[Formula size] 
  The size of a formula is defined inductively. Let 
  {\small $F$} be some arbitrary formula, and let 
  {\small $\formulasize(F)$} be the formula size of {\small $F$}. Then 
  it holds that; 
  \begin{itemize}
    \item {\small $\formulasize(F) = 1$} if 
      {\small $F \in \mathcal{S}$}. 
    \item {\small $\formulasize(F) = 
      \formulasize(F_1) + \formulasize(F_2) + 1$} 
      if {\small $F = F_1 \wedge F_2$}, 
      {\small $F = F_1 \vee F_2$}, or {\small $F = F_1 \gtrdot F_2$}. 
    \item {\small $\formulasize(F) = \formulasize(F_1) + 1$} 
      if {\small $F = \neg F_1$}. 
  \end{itemize}
\end{definition} 
\begin{definition}[Maximal number of {\small $\neg$} nestings]{\ }\\
  Given a formula {\small $F \in \mathfrak{F}$}, we denote
  by {\small $\negMax(F)$} a maximal number of {\small $\neg$} 
  nestings in {\small $F$}, whose definition goes as follows;  
  \begin{itemize}
      \item If {\small $F_0 = s$}, then {\small $\negMax(F_0) = 0$}.
      \item If {\small $F_0 = F_1 \wedge F_2$} or 
	{\small $F_0 = F_1 \vee F_2$} or 
	{\small $F_0 = F_1 \gtrdot F_2$}, then 
	{\small $\negMax(F_0) = max(\negMax(F_1), \negMax(F_2))$}. 
      \item If {\small $F_0 = \neg F_1$}, then  
	{\small $\negMax(F_0) = 1 + \negMax(F_1)$}. \\
    \end{itemize}
\end{definition} 
We now work on the main results. 
\begin{lemma}[Linking principle]
  Let {\small $F_1$} and {\small $F_2$} be two formulas 
  in unit chain expansion. Then it holds that 
  {\small $F_1 \gtrdot F_2$} has a reduction into 
  a formula in unit chain expansion.
  \label{linking_principle}
\end{lemma}   
\begin{IEEEproof} 
 In Appendix A. 
\end{IEEEproof}
\hide{
\begin{IEEEproof}  
  First apply {\small $\gtrdot$} reductions 2 and 3 on 
      {\small $F_1 \gtrdot F_2$} into a formula 
      in which the only occurrences of the chains are 
      {\small $f_0 \gtrdot 
      F_2$}, {\small $f_1 \gtrdot F_2$}, \dots, {\small $f_{k} 
      \gtrdot F_2$} for some {\small $k \in \mathbb{N}$} and 
      some {\small $f_0, f_1, \dots, f_k \in \mathfrak{U} 
      \cup \mathcal{S}$}. Then apply {\small $\gtrdot$} reductions 
      4 and 5 to each of those chains into a formula 
      in which the only occurrences of the chains are: 
      {\small $f_0 \gtrdot g_{0}, f_0 \gtrdot g_{1}, \dots, 
      f_0 \gtrdot g_{j}$},
      {\small $f_1 \gtrdot g_{0}$}, \dots, 
      {\small $f_1 \gtrdot g_{j}$}, \dots, 
      {\small $f_k \gtrdot g_{0}$}, \dots, {\small $f_k \gtrdot g_j$} 
      for some {\small $j \in \mathbb{N}$} and 
      some {\small $g_0, g_1, \dots, g_j \in \mathfrak{U}$}. 
      To each such chain, apply 
      {\small $\gtrdot$} reduction 1 as long as it is applicable. 
      This process cannot continue infinitely since any formula 
      is finitely constructed and since, under the premise, we 
      can apply induction
      on the number of elements of {\small $\mathcal{S}$} 
      occurring in {\small $g_x$}, {\small $0 \le x 
      \le j$}. 
      The straightforward inductive proof is left to readers. The 
      result is a formula in unit chain expansion. \\
\end{IEEEproof} 
}
\begin{lemma}[Reduction without negation]
  Any formula {\small $F_0 \in \mathfrak{F}$} in which 
  no {\small $\neg$} occurs reduces into 
  some formula in unit chain expansion.
  \label{normalisation_without_negation}
\end{lemma}
\begin{IEEEproof}   
  By induction on formula size. 
  For inductive cases, consider what 
  {\small $F_0$} actually is: 
  \begin{enumerate}
    \item {\small $F_0 = F_1 \wedge F_2$} or 
      {\small $F_0 = F_1 \vee F_2$}: Apply induction 
      hypothesis on {\small $F_1$} and {\small $F_2$}. 
    \item {\small $F_0 = F_1 \gtrdot F_2$}: Apply  
      induction hypothesis on {\small $F_1$} and {\small $F_2$} to 
      get {\small $F'_1 \gtrdot F'_2$} where {\small $F'_1$} and 
      {\small $F'_2$} are formulas in unit chain expansion. 
      Then apply Lemma \ref{linking_principle}.
        \end{enumerate} 
\end{IEEEproof}  
  \begin{lemma}[Reduction] 
    Any formula {\small $F_0 \in \mathfrak{F}$} reduces 
    into some formula in unit chain expansion. 
    \label{reduction_result}
  \end{lemma}  
  \begin{IEEEproof}  
    By induction on maximal number of {\small $\neg$} nestings 
    and a sub-induction on formula size. 
        Lemma \ref{normalisation_without_negation} for 
    base cases. Details are in Appendix B.  
    \hide{ For inductive cases, assume that the current 
    lemma holds true for all the formulas with {\small $\negMax(F_0)$} 
    of up to {\small $k$}. Then we conclude by showing that 
    it still holds true 
    for all the formulas with {\small $\negMax(F_0)$} of {\small $k+1$}. 
    Now, because any formula is finitely constructed,  
    there exist sub-formulas in which occur no {\small $\neg$}. 
    By Lemma \ref{normalisation_without_negation}, those sub-formulas 
    have a reduction into a formula in unit chain expansion. Hence 
    it suffices to show that those formulas 
    {\small $\neg F'$} with {\small $F'$} already in unit chain 
    expansion reduce into a formula in unit chain expansion, upon which
    inductive hypothesis applies for a conclusion. 
    Consider what {\small $F'$} is: 
    \begin{enumerate}
      \item {\small $s$}: then apply {\small $\neg$} reduction 1 
	on {\small $\neg F'$} to remove the {\small $\neg$}
	occurrence. 
      \item {\small $F_a \wedge F_b$}: apply {\small $\neg$} 
	reduction 2. 
	Then apply (sub-)induction hypothesis on 
	{\small $\neg F_a$} and {\small $\neg F_b$}.  
      \item {\small $F_a \vee F_b$}: apply {\small $\neg$} 
	reduction 3. Then apply (sub-)induction hypothesis on 
	{\small $\neg F_a$} and {\small $\neg F_b$}. 
      \item {\small $s \gtrdot F \in \mathfrak{U}$}: apply 
	{\small $\neg$} reduction 4. Then apply 
	(sub-)induction hypothesis on {\small $\neg F$}.
    \end{enumerate}  
}
\end{IEEEproof} 
\hide{
\begin{lemma}[Reduction into unit chains]
  Given any unit tree {\small $F \in \mathfrak{F}$}, there exists 
  a formula {\small $F_1 \in \mathfrak{F}$} in unit chain 
  expansion. Moreover, {\small $\recurseReduce(F)$} is 
  the unique reduction of {\small $F$}. 
  \label{unit_tree_reduction}
\end{lemma} 
\begin{IEEEproof}
  By rules in \textbf{Transformations}. \\
\end{IEEEproof}    
}
\begin{lemma}
  For any {\small $F \in \mathfrak{F}$} expanded in unit chains, there 
  exists 
  {\small $v \in \{0,1\}$} such that {\small $[\intFrame \models
      F] = v$} for  
  any valuation frame. 
  \label{simple_lemma2}
\end{lemma}   
\begin{IEEEproof}
  Since 
  a value 0/1 is assignable to any element of 
  {\small $\mathcal{S} \cup \mathfrak{U}$} by Definition 
  \ref{model}, 
  it is (or they are if more than 
  one in \{0, 1\}) assignable to {\small $[\mathfrak{M} \models
  F]$}. \\
\end{IEEEproof} 
Hence we obtain the desired result for the first objective.
\begin{proposition}
  To any {\small $F \in \mathfrak{F}$} corresponds  
  at least one formula {\small $F_a$} in unit chain expansion 
  into which {\small $F$} reduces. 
  It holds for any such {\small $F_a$} that 
  {\small $[\mathfrak{M} \models F_a] \in 
      \{0, 1\}$} for 
  any valuation frame. \\
  \label{simple_proposition} 
\end{proposition}    
For the next sub-section, 
the following observation about negation on
a unit chain comes in handy. Let us state a procedure. 
\begin{definition}[Procedure \recurseReduce]{\ }\\
  The procedure given below 
  takes as an input a formula {\small $F$} in unit chain expansion. \\
\textbf{Description of {\small $\recurseReduce(F)$}}
\begin{enumerate}
    \item Replace {\small $\wedge$} in {\small $F$} 
      with {\small $\vee$}, and 
      {\small $\vee$} with {\small $\wedge$}. These two 
      operations are simultaneous. 
    \item Replace all the non-chains {\small $s \in \mathcal{S}$} 
      in {\small $F$} simultaneously with 
      {\small $s^c\ (\in \mathcal{S})$}.   
    \item For every chain {\small $F_a$} in {\small $F$} with 
      its head {\small $s \in \mathcal{S}$} for some 
      {\small $s$} and its tail 
      {\small $F_{\code{tail}}$}, replace {\small $F_a$}  
      with {\small $(s^c \vee (s \gtrdot 
      (\recurseReduce(F_{\code{tail}}))))$}.    
    \item Reduce {\small $F$} via {\small $\gtrdot$} reductions 
      in unit chain expansion. 
  \end{enumerate}
\end{definition}
Then we have the following result. 
\begin{proposition}[Reduction of negated unit chain expansion] 
  Let {\small $F$} be a formula in unit chain expansion. Then 
  {\small $\neg F$} reduces 
  via the {\small $\neg$} and {\small $\gtrdot$} reductions
  into {\small $\recurseReduce(F)$}. Moreover 
  {\small $\recurseReduce(F)$} is the unique reduction 
  of {\small $\neg F$}.  
  \vspace{-4mm}
    \label{special_reduction}
\end{proposition}  
\begin{IEEEproof} 
  For the uniqueness, observe that only  
  {\small $\neg$} reductions and 
  {\small $\gtrdot$} reduction 5 are used 
  in reduction of {\small $\neg F$}, and that  
  at any point during the reduction, 
  if there occurs a sub-formula in the form {\small $\neg F_x$}, 
  the sub-formula {\small $F_x$} cannot be reduced by any 
  reduction rules. Then the proof of the uniqueness is 
  straightforward. \\
\end{IEEEproof}
\subsubsection{Unit chain expansions form Boolean 
algebra}{\ }\\
We make use of disjunctive normal form 
in this sub-section for simplification of proofs. 
\begin{definition}[Disjunctive/Conjunctive normal form]
  A formula {\small $F \in \mathfrak{F}$} is 
  defined to be in disjunctive normal form only if    
  {\small $\exists i,j,k \in \mathbb{N}\  
  \exists h_{0}, \cdots, h_i \in \mathbb{N}\
  \exists 
  f_{00}, \dots, f_{kh_k} \in 
  \mathfrak{U} \cup \mathcal{S}.F = \vee_{i =0}^k \wedge_{j = 0}^{h_i} 
  f_{ij}$}.
 Dually, a formula {\small $F \in \mathfrak{F}$} 
  is defined to be in conjunctive normal form 
  only if {\small $\exists i, j, k \in \mathbb{N}\ \exists  
  h_0, \cdots, h_i \in \mathbb{N}\
  \exists 
  f_{00}, \dots, f_{kh_k} \in 
  \mathfrak{U} \cup \mathcal{S}.F = \wedge_{i=0}^k \vee_{j =0}^{h_i} f_{ij}$}. \\
\end{definition}      
Now, for the second objective of ours, we prove 
that {\small $\mathfrak{U} \cup \mathcal{S}$}, 
{\small $\recurseReduce$}, {\small $\vee^{\dagger}$} and {\small $\wedge^{\dagger}$} 
form a Boolean algebra (\emph{Cf.} \cite{wikiBooleanAlgebra} for the laws 
of Boolean algebra), from which follows 
the required outcome.          
\begin{proposition}[Annihilation/Identity]
  For any formula {\small $F$}   
  in unit chain expansion and 
  for any valuation frame, 
  it holds (1) that {\small $[\intFrame \models \top \wedge 
  F] = [\intFrame \models F]$}; 
  (2) that {\small $[\intFrame \models \top \vee 
  F] = [\intFrame \models \top]$};
  (3) that 
  {\small $[\intFrame \models \bot \wedge F] =
  [\intFrame \models \bot]$}; 
  and (4) that  
  {\small $[\intFrame \models \bot \vee F] =
  [\intFrame \models F]$}. 
\end{proposition}   
\begin{lemma}[Elementary complementation]  
  For any {\small $s_0 \gtrdot s_1 \gtrdot \dots \gtrdot s_k
  \in \mathfrak{U} \cup \mathcal{S}$} for some 
  {\small $k \in \mathbb{N}$}, 
  if for a given valuation frame 
  it holds that 
  {\small $[\mathfrak{M} \models
  s_0 \gtrdot s_1 \gtrdot \dots \gtrdot s_{k}] = 1$}, 
  then it also holds that 
  {\small $[\intFrame \models \recurseReduce(s_0 
  \gtrdot s_1 \gtrdot \dots 
  \gtrdot s_{k})] = 0$}; or 
  if it holds that 
  {\small $[\intFrame \models s_0 \gtrdot s_1 \gtrdot \dots 
  \gtrdot s_{k}] = 0$}, then 
  it holds that 
  {\small $[\intFrame \models \recurseReduce(s_0 \gtrdot s_1 \gtrdot \dots
  \gtrdot s_{k})] = 1$}. These two events are mutually 
  exclusive.
  \label{unit_chain_excluded_middle}
\end{lemma}
\begin{IEEEproof}              
    In Appendix C.  
    \hide{
 For the first one,
 {\small $[(\mathsf{I}, \mathsf{J})\models s_0 \gtrdot 
 s_1 \gtrdot \dots \gtrdot s_{k}] = 1$} implies that
 {\small $\mathsf{I}(|\epsilon |, s_0)\! =\! 
 \mathsf{I}(|s_0|, s_1) \!=\! \dots \!=\! 
 \mathsf{I}(| s_0.s_1.\dots.s_{k - 1}|, s_{k})\! =\! 1$}.
 So we have; 
 {\small $\mathsf{I}(| \epsilon |, s_0^c) = 
 \mathsf{I}(| s_0 |, s_1^c) = \dots = 
 \mathsf{I}(| s_0.s_1\dots.s_{k - 1} |, s_{k}^c) = 0$} by the 
 definition of {\small $\mathsf{I}$}. 
 Meanwhile, 
 {\small $\recurseReduce(s_0 \gtrdot s_1 \gtrdot \cdots \gtrdot 
 s_k) = s_0^c \vee (s_0 \gtrdot ((s_1^c \vee (s_1 \gtrdot \cdots)))) 
 = s_0^c \vee (s_0 \gtrdot s_1^c) \vee (s \gtrdot s_1 \gtrdot 
 s_2^c) \vee \cdots \vee 
 (s \gtrdot s_1 \gtrdot \cdots \gtrdot s_{k-1} \gtrdot s_k^c)$}. 
 Therefore {\small $[(\mathsf{I}, \mathsf{J})\models
 \recurseReduce(s_0 \gtrdot s_1 \gtrdot \cdots \gtrdot s_k)]  
 = 0 \not= 1$} for the given interpretation frame. \\
 \indent For the second obligation, 
 {\small $[(\mathsf{I}, \mathsf{J})\models
 s_0 \gtrdot s_1 \gtrdot \dots \gtrdot s_{k}] = 0$} 
 implies that   
 {\small $[\mathsf{I}(|\epsilon|, s_0) 
 = 0] \vee^{\dagger} [\mathsf{I}(|s_0|, s_1) = 0] 
 \vee^{\dagger} \dots \vee^{\dagger} [\mathsf{I}(|s_0.s_1.\dots.
 s_{k -1}|, s_{k}) = 0]$}. Again 
 by the definition of {\small $\mathsf{I}$}, 
 we have the required result. That these two events 
 are mutually exclusive is trivial. \\ 
    }
\end{IEEEproof}  
\begin{proposition}[Associativity/Commutativity/Distributivity]
  Given any formulas {\small $F_1, F_2, F_3 \in \mathfrak{F}$} 
  in unit chain expansion and any valuation frame  
  {\small $\mathfrak{M}$}, the following hold:  
  \begin{enumerate}
    \item  {\small $[\intFrame \models F_1] \wedge^{\dagger} 
  ([\intFrame \models F_2] \wedge^{\dagger} 
  [\intFrame \models F_3]) = ([\intFrame \models F_1] \wedge^{\dagger} 
  [\intFrame \models F_2]) \wedge^{\dagger} [\intFrame \models F_3]$} (associativity 1). 
\item {\small $[\intFrame \models F_1] \vee^{\dagger} 
  ([\intFrame \models F_2] \vee^{\dagger} [\intFrame \models F_3]) 
  = ([\intFrame \models F_1] \vee^{\dagger} [\intFrame \models F_2]) \vee^{\dagger} F_3$} 
  (associativity 2). 
\item {\small $[\intFrame \models F_1] \wedge^{\dagger} [\intFrame \models F_2] 
  = [\intFrame \models F_2] \wedge^{\dagger} [\intFrame \models F_1]$} (commutativity 1). 
\item {\small $[\intFrame \models F_1] \vee^{\dagger} [\intFrame \models F_2] 
  = [\intFrame \models F_2] \vee^{\dagger} [\intFrame \models F_1]$} (commutativity 2).  
\item {\small $[\intFrame \models F_1] \wedge^{\dagger} 
  ([\intFrame \models F_2] \vee^{\dagger} [\intFrame \models F_3]) = 
  ([\intFrame \models F_1] \wedge^{\dagger} [\intFrame \models F_2]) \vee^{\dagger} 
  ({[\intFrame \models F_1]} \wedge^{\dagger} [\intFrame \models F_3])$} (distributivity 1).  
\item {\small $[\intFrame \models F_1] \vee^{\dagger} 
  ([\intFrame \models F_2] \wedge^{\dagger} [\intFrame \models F_3]) = 
  ([\intFrame \models F_1] \vee^{\dagger} [\intFrame \models F_2]) \wedge^{\dagger} 
  ({[\intFrame \models F_1]} \vee^{\dagger} [\intFrame \models F_3])$} (distributivity 2).  
  \end{enumerate} 
  \label{associativity_commutativity_distributivity}
\end{proposition}  
\begin{IEEEproof}       
    Make use of Lemma \ref{unit_chain_excluded_middle}. Details are in Appendix D. 
    \hide{
  Let us generate a set of 
  expressions finitely constructed from the following grammar;\\
  \indent {\small $X := [(\mathsf{I}, \mathsf{J})\models f] \ | \ X \wedge^{\dagger} X 
  \ | \ X \vee^{\dagger} X$} where {\small $f \in \mathfrak{U} \cup 
  \mathcal{S}$}. \\
  Then first of all it is straightforward to 
  show that {\small $[(\mathsf{I}, \mathsf{J})\models F_i] = X_i$} 
  for each {\small $i \in \{1,2,3\}$} for some {\small $X_1, X_2, X_3$} 
  that the above grammar recognises. By Lemma \ref{unit_chain_excluded_middle} each expression ({\small $[(\mathsf{I}, \mathsf{J})\models f_x]$} for 
  some {\small $f_x \in \mathfrak{U} \cup \mathcal{S}$}) 
  is assigned one and only one value {\small $v \in \{0,1\}$}. Then 
  since {\small $1 \vee^{\dagger} 1 = 1 \vee^{\dagger} 0 = 
  0 \vee^{\dagger} 1 = 1$}, 
  {\small $0 \wedge^{\dagger} 0 = 0 \wedge^{\dagger} 1 = 
  1 \wedge^{\dagger} 0 = 0$}, and 
  {\small $1 \wedge^{\dagger} 1 = 1$} by definition (given at 
  the beginning 
  of this section), 
  it is also the case that {\small $[(\mathsf{I}, \mathsf{J})\models F_i]$} 
  is assigned one and only one value {\small $v_i \in  \{0,1\}$} 
  for each {\small $i \in \{1,2,3\}$}. Then the proof for the 
  current proposition is straightforward. \\ 
    }
\end{IEEEproof} 
\begin{proposition}[Idempotence and Absorption]
  Given any formula {\small $F_1, F_2 \in \mathfrak{F}$} 
  in unit chain expansion, 
  for any valuation frame it holds that 
  {\small $[\intFrame \models F_1] \wedge^{\dagger} [\intFrame \models F_1] 
  = {[\intFrame \models F_1]} \vee^{\dagger} [\intFrame \models F_1] = [\intFrame \models F_1]$} 
  (idempotence); and 
  that {\small $[\intFrame \models F_1] \wedge^{\dagger} 
  ([\intFrame \models F_1] \vee^{\dagger} {[\intFrame \models F_2]}) 
  = [\intFrame \models F_1] \vee^{\dagger} 
  ([\intFrame \models F_1] \wedge^{\dagger} [\intFrame \models F_2]) = [\intFrame \models F_1]$} 
  (absorption). 
  \label{idempotence_absorption}
\end{proposition}
\begin{IEEEproof}  
  Both {\small $F_1, F_2$} are assigned one and only one value 
  {\small $v \in \{0,1\}$} (\emph{Cf}. Appendix D). Trivial to verify. \\
\end{IEEEproof} 
We now prove laws involving \recurseReduce. 
\begin{lemma}[Elementary double negation]
  Let {\small $F$} denote\linebreak {\small $s_0 \gtrdot s_1 \gtrdot \cdots \gtrdot 
  s_k \in \mathfrak{U} \cup \mathcal{S}$} for 
  some {\small $k \in \mathbb{N}$}.
  Then 
  for any valuation frame it holds 
  that {\small $[\intFrame \models F] 
  = [\intFrame \models \recurseReduce(\recurseReduce(F))]$}. 
  \label{unit_double_negation}
\end{lemma} 
\begin{IEEEproof}  
    {\small $\recurseReduce(\recurseReduce(F))$} is in 
    conjunctive normal form. Transform this to 
    disjunctive normal form, and observe that almost 
    all the clauses are assigned 0. 
    Details are in Appendix E. 
\end{IEEEproof} 
\begin{proposition}[Complementation/Double negation]{\ }\\
  For any {\small $F$} in unit chain expansion 
  and for any valuation frame, it holds that 
  {\small $1 = [\intFrame \models F \vee 
  \recurseReduce(F)]$} 
   and that 
  {\small $0 = [\intFrame \models F \wedge \recurseReduce(F)]$} 
  (complementation). 
  Also, for any {\small $F \in \mathfrak{F}$} in unit chain 
  expansion and 
  for any valuation frame 
  it holds that {\small $[\intFrame \models F] 
  = [\intFrame \models \recurseReduce(\recurseReduce(F))]$} (double negation). 
\label{excluded_middle}
\end{proposition} 
\begin{IEEEproof}          
    Make use of disjunctive normal form, Lemma \ref{unit_chain_excluded_middle} and 
    Lemma \ref{unit_double_negation}. 
    Details are in Appendix F. 
\end{IEEEproof}      
\begin{theorem} 
  Denote by {\small $X$} 
  the set of the expressions comprising all 
  {\small $[\intFrame \models f_x]$} for 
  {\small $f_x \in \mathfrak{U} \cup \mathcal{S}$}. 
  Then for every valuation frame, {\small $(X, \recurseReduce, \wedge^{\dagger}, \vee^{\dagger})$} 
  defines a Boolean algebra. 
  \label{theorem_1}
\end{theorem}
\begin{IEEEproof}
  Follows from earlier propositions and lemmas.  \\
\end{IEEEproof}
\subsubsection{Gradual classical logic is neither para-consistent 
    nor inconsistent 
    }{\ }\\ 
To achieve the last objective
we assume two notations. 
\begin{definition}[Sub-formula notation]
Given a formula {\small $F \in \mathfrak{F}$}, 
  we denote by {\small $F[F_a]$} the fact that 
  {\small $F_a$} occurs as a sub-formula in {\small $F$}.   
  Here the definition of a sub-formula of a formula 
  follows one that is found in standard textbooks on logic 
  \cite{Kleene52}. 
  {\small $F$} itself is a sub-formula of {\small $F$}. 
\end{definition}
\begin{definition}[Small step reductions]
  By {\small $F_1 \leadsto F_2$} for  
  some formulas {\small $F_1$} and {\small $F_2$} we denote
  that {\small $F_1$} 
  reduces in one reduction step into {\small $F_2$}. By 
  {\small $F_1 \leadsto_{r} F_2$} we denote that 
  the reduction holds explicitly by 
  a reduction rule {\small $r$} (which is either of the 
  7 rules). By {\small $F_1 \leadsto^* F_2$} we denote 
  that {\small $F_1$} reduces 
  into {\small $F_2$} in a finite number of steps including 
  0 step in which case {\small $F_1$} is said to be 
  irreducible. By {\small $F_1 \leadsto^k F_2$} we denote 
  that the reduction is in exactly {\small $k$} steps. 
  By 
  {\small $F_1 \leadsto^*_{\{r_1, r_2, \cdots\}} F_2$} or 
  {\small $F_1 \leadsto^k_{\{r_1, r_2, \cdots\}} F_2$} we denote 
  that the reduction is via those specified rules 
  {\small $r_1, r_2, \cdots$} only. \\
\end{definition}       
Along with them, we also enforce that  
{\small $\mathcal{F}(F)$} denote the set of 
formulas in unit chain expansion that {\small $F \in \mathfrak{F}$}  
can reduce into. A stronger result than Lemma \ref{normalisation_without_negation} 
follows. 
\begin{lemma}[Bisimulation without negation]  
  Assumed below are pairs of formulas in which 
  {\small $\neg$} does not occur. 
  {\small $F'$} differs from {\small $F$} only by
  the shown sub-formulas, \emph{i.e.} {\small $F'$} 
  derives from {\small $F$} by replacing the shown sub-formula 
  for {\small $F'$} 
  with the shown sub-formula for {\small $F$} and vice versa. 
  Then for each pair  
  {\small $(F, F')$} 
  below, it holds for every valuation frame 
  that 
  {\small $[\intFrame \models F_1] = [\intFrame \models F_2]$} for 
  all {\small $F_1 \in \mathcal{F}(F)$} and 
  for all {\small $F_2 \in \mathcal{F}(F')$}.   
  {\small 
  \begin{eqnarray}\nonumber 
    F[(F_a \wedge F_b) \gtrdot F_c] &,& F'[(F_a \gtrdot F_c) 
    \wedge (F_b \gtrdot F_c)]\\\nonumber
    F[(F_a \vee F_b) \gtrdot F_c] &,& F'[(F_a \gtrdot F_c) 
    \vee (F_b \gtrdot F_c)]\\\nonumber 
    F[F_a \gtrdot (F_b \wedge F_c)] &,& F'[(F_a \gtrdot F_b)
    \wedge (F_a \gtrdot F_c)]\\\nonumber 
    F[F_a \gtrdot (F_b \vee F_c)] &,& F'[(F_a \gtrdot F_b)
    \vee (F_a \gtrdot F_c)]\\\nonumber 
    F[(F_a \gtrdot F_b) \gtrdot F_c] &\!\!\!\!\!\!\!\!\!,& \!\!\!\!\!\!\!\!\!\!
    F'[(F_a \gtrdot F_c) \wedge ((F_a \gtrdot F_b) \vee (F_a \gtrdot 
    F_b \gtrdot F_c))]
  \end{eqnarray}  
  }
  \label{bisimulation}
\end{lemma} 
\begin{IEEEproof} 
    By induction on the number of reduction steps and a sub-induction 
    on formula size in each direction of bisimulation. Details 
    are in Appendix G. 
\end{IEEEproof} 
\begin{lemma}[Other bisimulations] 
  For each pair {\small $(F \in \mathfrak{F}, F' \in \mathfrak{F})$} below,
it holds for every valuation frame  
(1) that
{\small $\forall F_1 \in \mathcal{F}(F).\exists F_2 
\in \mathcal{F}(F').[\intFrame \models F_1] = 
[\intFrame \models F_2]$} 
and (2) that {\small $\forall F_2 \in \mathcal{F}(F').\exists F_1 
\in \mathcal{F}(F).[\intFrame \models F_1] = 
[\intFrame \models F_2]$}. 
  Once again, 
  {\small $F$} and {\small $F'$} differ only by 
  the shown sub-formulas.  
  {\small 
  \begin{eqnarray}\nonumber 
    F[\neg (F_a \wedge F_b)] &,& F'[\neg F_a \vee \neg F_b]\\\nonumber
    F[\neg (F_a \vee F_b)] &,& F'[\neg F_a \wedge \neg F_b]\\\nonumber 
    F[s \vee s] &,& F'[s]\\\nonumber
    F[s \vee F_a \vee s] &,& F'[s \vee F_a]\\\nonumber
    F[s \wedge s] &,& F'[s]\\\nonumber
    F[s \wedge F_a \wedge s] &,& F'[s \wedge F_a]\\\nonumber
    F[s^c] &,& F'[\neg s]
  \end{eqnarray} 
  }
  \label{other_bisimulation}
\end{lemma}
\begin{IEEEproof} 
    By simultaneous induction on the number of reduction steps 
    and a sub-induction on formula size. Details are in Appendix H. 
\end{IEEEproof}
\begin{lemma}[Normalisation without negation]  
  Given a formula {\small $F \in \mathfrak{F}$}, 
  if 
  {\small $\neg$} does not occur in {\small $F$}, then  
  it holds 
  for every valuation frame either that 
  {\small $[\mathfrak{M} \models F_a] = 1$} 
  for all {\small $F_a \in \mathcal{F}(F)$} or else 
  that 
  {\small $[\intFrame \models F_a] = 0$}  
for all {\small $F_a \in \mathcal{F}(F)$}. 
    \label{reduction_without_negation}
\end{lemma}
\begin{IEEEproof} 
  Consequence of Lemma \ref{bisimulation}.  
\end{IEEEproof} 
  \begin{theorem}[Normalisation]
    Given a formula {\small $F \in \mathfrak{F}$}, 
    denote the set of formulas in unit chain expansion 
    that it can reduce into by {\small $\mathcal{F}_1$}. 
    Then it holds 
    for every valuation frame 
    either that {\small $[\intFrame \models F_a] = 1$} 
    for all {\small $F_a \in \mathcal{F}_1$} or else 
    that 
    {\small $[\intFrame \models F_a] = 0$} for all {\small $F_a \in 
    \mathcal{F}_1$}.
    \label{theorem_normalisation}
  \end{theorem}
  \begin{IEEEproof}    
    By induction on maximal number of {\small $\neg$} nestings 
    and a sub-induction on formula size. We quote Lemma 
    \ref{reduction_without_negation} for base cases. Details 
    are in Appendix I. \\ 
\end{IEEEproof}
By the result of Theorem \ref{theorem_1} and Theorem \ref{theorem_normalisation}, we may 
define implication: {\small $F_1 \supset F_2$} to be 
an abbreviation of {\small $\neg F_1 \vee F_2$} - {\it exactly the same} -
as in classical logic.   
\section{Decidability}  
We show a decision procedure {\small $\oint$} for 
universal validity 
of some input formula {\small $F$}. Here, 
{\small $z: Z$} for some 
{\small $z$} and {\small $Z$} denotes a variable {\small $z$} 
of type {\small $Z$}. 
   Also assume a terminology of `object level', which is defined inductively. 
Given {\small $F$} in unit chain expansion, (A) if
   {\small $s \in \mathcal{S}$} in {\small $F$} occurs 
   as a non-chain or 
   as a head of a unit chain, then 
   it is said to be at the 0-th object level. 
   (B) if it 
   occurs in a unit chain 
   as {\small $s_0 \gtrdot \dots \gtrdot s_k \gtrdot s$} or 
   as {\small $s_0 \gtrdot \dots \gtrdot s_k  \gtrdot 
       s \gtrdot ...$}
   for some {\small $k \in \mathbb{N}$} 
   and some {\small $s_0, \dots, s_k \in \mathcal{S}$}, 
   then it is said to be at the (k+1)-th object level. 
   Further, assume a function {\small $\code{toSeq}: 
       \mathbb{N} \rightarrow \mathcal{S}^*$} satisfying
   {\small $\code{toSeq}(0) = \epsilon$} 
   and {\small $\code{toSeq}(k+1) = \underbrace{\top.
           \dots.\top}_{k+1}$}. 
\begin{description}
    \item[{\small $\oint(F: \mathfrak{F},   
            \code{object}\_\code{level} : \mathbb{N} 
            )$}]{\ }\\ \textbf{returning 
            either 0 or 1}\\
    $\backslash\backslash$ This pseudo-code uses 
    {\small $n, o:\mathbb{N}$}, {\small $F_a, 
        F_b:\mathfrak{F}$}. \\
    \textbf{L0: } Duplicate {\small $F$} and 
    assign the copy to {\small $F_a$}. 
    If {\small $F_a$} is not already
    in unit chain expansion, then reduce it into  
    a formula in unit chain expansion. \\ 
\textbf{L1: } {\small $F_b := \squash(F_a, \code{object}\_\code{level})$}. \\  
\textbf{L2: }  {\small $n := \code{COUNT}\_\code{DISTINCT}(F_b)$}.\\   
\textbf{L3$_0$: } For each {\small $\mathsf{I}: 
    \code{toSeq}(\code{object}\_\code{level}) \times \mathcal{S}$} distinct
for the {\small $n$} elements of {\small $\mathcal{S}$} 
at the given object level, 
    Do:  \\  
    \textbf{L3$_1$: } If {\small $\sat(F_b, 
        \mathsf{I})$}, 
    then go to \textbf{L5}.\\   
    \textbf{L3$_2$: } Else if no 
    unit chains occur in {\small $F_a$}, 
    go to \textbf{L3}$_5$.\\ 
    \textbf{L3$_3$: } 
    {\small $o := \oint(\code{REWRITE}(F_a, \mathsf{I}, 
        \code{object}\_\code{level}),$}\\ 
        {\small $\code{object}\_\code{level} + 1)$}. \\ 
    \textbf{L3$_4$: } 
    If {\small $o = 0$}, go to \textbf{L5}. \\
    \textbf{L3$_5$: } End of For Loop. \\
    \textbf{L4: } return 1. $\backslash\backslash$ Yes.\\
    \textbf{L5: } return 0. $\backslash\backslash$ No.\\
\end{description}    
\begin{description}
    \item[\squash({\small $F: \mathfrak{F}, \code{object}\_\code{level}: \mathbb{N}$}) returning 
    {\small $F': \mathfrak{F}$}]
    {\ }\\ 
    \textbf{L0}: {\small $F' := F$}.  \\
    \textbf{L1}: For every 
    {\small $s_0 \gtrdot s_1 \gtrdot \dots \gtrdot s_{k}$} 
    for some {\small $k \in \mathbb{N}$} greater than 
    or equal to \code{object}\_\code{level} 
    and 
    some {\small $s_0, s_1, \dots, s_{k} \in \mathcal{S}$} 
    occurring 
    in {\small $F'$}, replace it with 
    {\small $s_0 \gtrdot \dots \gtrdot s_{\code{object}\_
            \code{level}}$}. \\
    \textbf{L2}: return {\small $F'$}. 
\end{description} 
\begin{description}
  \item[$\code{COUNT}\_\code{DISTINCT}(F : \mathfrak{F})$ returning  
    {\small $n : \mathbb{N}$}]  
    {\ } \\
    \textbf{L0}: 
    return {\small $n:=$} (number of distinct members 
    of {\small $\mathcal{A}$} in {\small $F$}
     ). 
\end{description}
\begin{description}
  \item[\sat({\small $F: \mathfrak{F}, \mathsf{I}: \mathsf{I}$}) returning 
      \code{true} or \code{false}]{\ }\\ 
      \textbf{L0}: return \code{true} if, 
    for the given interpretation {\small $\mathsf{I}$},\linebreak
    {\small $[(\mathsf{I}, \mathsf{J}) \models F] = 0$}. 
    Otherwise, return \code{false}. 
\end{description} 
\begin{description}
    \item[\rewrite({\small $F: \mathfrak{F}, \mathsf{I}: \mathsf{I}, \code{object}\_\code{level}: \mathbb{N}$}) returning 
    {\small $F': \mathfrak{F}$}]{\ }\\    
    \textbf{L0}: {\small $F' := F$}.  \\
    \textbf{L1}: remove
    all the non-unit-chains and 
    unit chains shorter than or equal to 
    \code{object}\_\code{level} from {\small $F'$}. The 
    removal is in the following sense: 
      if {\small $f_x \wedge F_x$}, {\small $F_x \wedge f_x$}, 
	{\small $f_x \vee F_x$} or {\small $F_x \vee f_x$} 
	occurs as a sub-formula in {\small $F'$} for  
    {\small $f_x$} those just specified, then 
	replace them not simultaneously but one at a time 
	to {\small $F_x$} until no more reductions are possible. \\
    \textbf{L2$_0$}: For each unit chain 
    {\small $f$} in {\small $F'$}, Do:\\
    \textbf{L2$_1$}: if the head of {\small $f$} is 0 under 
    {\small $\mathsf{I}$}, then remove the unit chain 
    from {\small $F'$}; else replace the head of {\small $f$} with {\small $\top$}.  \\
    \textbf{L2$_2$}: End of For Loop.  \\
    \textbf{L3}: return {\small $F'$}. 
\end{description}  
The intuition of the procedure is found within the 
proof below. 
\begin{proposition}[Decidability of gradual classical logic] 
Complexity of {\small $\oint(F, 0)$} is at most \exptime.  
\end{proposition}       
\begin{IEEEproof}     
   We show that it is a decision procedure.  
   That the complexity bound cannot be worse 
   than {\exptime} is clear from 
   the semantics (for \textbf{L0}) and from 
   the procedure itself. 
   Consider \textbf{L0} of the main procedure. 
   This reduces a given formula into 
   a formula in unit chain expansion.       
   In \textbf{L1} of the main procedure, 
   we get a snapshot of the input formula. We
   extract from it components of the 0-th object level, 
   and check if it is (un)satisfiable. The motivation  
   for this operation 
   is as follows: if the input formula is contradictory 
   at the 0th-object level, the input formula is 
   contradictory by the definition of {\small $\mathsf{J}$}. 
   Since we are considering validity of a formula, 
   we need to check all the possible valuation frames. 
   The number is determined by distinct {\small $\mathcal{A}$}
   elements. \textbf{L2} gets the number (n). 
   The For loop starting at \textbf{L3}$_0$ iterates through the {\small $2^n$} 
   distinct interpretations. If the snapshot 
   is unsatisfiable for any such valuation frame, 
   it cannot be valid, which in turn implies 
   that the input formula cannot be valid (\textbf{L3}$_1$). 
   If the snapshot is satisfiable and if 
   the maximum object-level in the input formula 
   is the 0th, \emph{i.e.} the snapshot is the input 
   formula, then the input formula 
   is satisfiable for this particular valuation frame, 
   and so we check the remaining valuation frames
   (\textbf{L3}$_2$). Otherwise, if 
   it is satisfiable and if the maximum object-level
   in the input formula is not the 0th, then 
   we need to check that snapshots in all the other 
   object-levels of the input formula are satisfiable 
   by all the valuation frames. We do this check 
   by recursion (\textbf{L3}$_3$). Notice the first
   parameter {\small $\code{REWRITE}(F_a, \mathsf{I},
       \code{object}\_\code{level})$} here. 
   This returns some formula {\small $F'$}. At the 
   beginning of the sub-procedure, {\small $F'$} is 
   a duplicated copy of {\small $F_a$} (not 
   {\small $F_b$}). Now, 
   under the particular 0-th object level interpretation 
   {\small $\mathsf{I}$}, some unit chain in 
   {\small $F_a$} may be 
   already evaluated to 0. Then 
   we do not need consider them 
   at any deeper object-level. 
   So we remove
   them from {\small $F'$}. Otherwise, 
   in all the remaining unit chains, the 0-th object 
   gets local interpretation of 1. So we replace 
   the {\small $\mathcal{S}$} element at the 0-th object 
   level with {\small $\top$} which always gets 1. 
   Finally, all the non-chain 
   {\small $\mathcal{S}$} constituents 
   and all the chains shorter than or equal to 
   \code{object}\_\code{level} in {\small $F_a$} 
   are irrelevant at a higher object-level. So we 
   also remove them (from {\small $F'$}). We pass this 
   {\small $F'$} and an incremented \code{object}\_\code{level} to the main procedure for 
   the recursion. \\
   \indent The recursive process continues either until  
   a sub-formula passed to the main procedure 
   turns out to be invalid, in which case  
   the recursive call returns 0 
   (\textbf{L2}$_2$ and \textbf{L4} in the
   main procedure) 
   to the caller who assigns 0 to $o$ 
   (\textbf{L2}$_4$) and again returns 0, and so on 
   until the first recursive caller. 
   The caller receives 0 once again to conclude that 
   {\small $F$} is invalid, as expected. Otherwise, 
   we have that {\small $F$} is valid, for we 
   considered all the valuation frames. 
   The number of recursive calls cannot be infinite.\\
\end{IEEEproof} 
\section{Conclusion and Related Thoughts}   


There are many existing logics to which 
gradual classical logic can relate, including ones 
below. ``G(g)radual classical logic'' is abbreviated by \code{Grad}. 
\subsection{Para-consistent Logic} 
In classical logic a contradictory statement implies just anything 
expressible in the given domain of discourse. Not so in the family 
of para-consistent logics where it is 
distinguished from other forms of inconsistency \cite{Marcos05}; 
what is trivially the case in classical logic, 
say {\small $a_1 \wedge 
a_1^c \supset a_2$} for any propositions {\small $a_1$} 
and {\small $a_2$}, is not an axiom. Or, if  
my understanding about them is sufficient, it 
actually holds in the sense that to each contradiction expressible 
in a para-consistent logic associates  
a sub-domain of discourse within which it entails anything; 
however, just as \code{Grad} internalises classical logic, 
so do para-consistent logics, revealing the extent of the explosiveness
of contradiction within them. In some sense para-consistent logics 
model parallel activities as seen in concurrency. What 
\code{Grad} on the other hand aims to model is conceptual scoping. 
As they do not pose an active conflict to each other, 
it should 
be possible to derive an extended logic which benefits from both features. 
\subsection{Epistemic Logic/Conditional Logic}   
Epistemic logic concerns knowledge and belief,
augmenting propositional logic with epistemic operators {\small $K_c$} 
for knowledge and {\small $B_c$} for belief
such that {\small $K_c a$}/{\small $B_c a$} means that 
a proposition {\small $a$} is known/believed to be true 
by an agent {\small $c$}. \cite{Hendricks06}. 
\code{Grad} has a strong link to knowledge and belief, being 
inspired by tacit agreement on 
assumptions about attributed objects. To seek a correspondence, we may 
tentatively assign to {\small $a_0 \gtrdot a_1$} 
a mapping of {\small $a_0 \wedge K_c/B_c a_1$}. 
However, this mapping is not very adequate due to the fact that {\small $K_c/B_c$} enforces 
a global sense of 
knowledge/belief that does not update in the course of 
discourse. 
The relation that {\small $\gtrdot$} expresses between 
{\small $a_0$} and {\small $a_1$} is not captured this way. A more proximate mapping 
is achieved with the conditional operator {\small $>$} in conditional logics \cite{Horacio13} with which we may map 
{\small $a_0 \gtrdot a_1$} into {\small $a_0 \wedge 
    (a_0 > a_1)$}. But by this mapping
the laws of {\small $>$} will no longer follow any of normal, classical, monotonic or 
regular (\emph{Cf}. \cite{Chellas80} or Section 3 in 
\cite{Horacio13}; note that the small letters
{\small $a, b, c, ...$} 
in the latter reference are not literals but propositional formulas)
conditional logics'. RCEA holds safely, but all the rest: 
RCEC; RCM; RCR and RCK fail since availability of some 
{\small $b$} and {\small $c$} equivalent in one sub-domain of 
discourse of \code{Grad} does not imply their equivalence 
in another sub-domain. Likewise, 
the axioms listed in Section 3 of \cite{Horacio13} fail 
save CC (understand it by {\small $a \wedge (a > b)  
    \wedge a \wedge (a > c) \supset a \wedge (a > b \wedge c)$}), CMon 
and CM. Further studies should be useful in order to unravel 
a logical perspective into how some facts that act as 
pre-requisites for others could affect knowledge and belief. 
\subsection{Intensional Logic/Description Logic}  
Conditional logics were motivated by counterfactuals \cite{Plumpton31,Lewis01}, 
\emph{e.g.} ``If X were the case, then Y would be the 
case." According to the present comprehension of the author's 
about reasoning about  
such statements as found in Appendix J in the form of an informal essay, the 
reasoning process involves transformation of one's consciousness 
about 
the antecedent that he/she believes is impossible. However, even if 
we require the said transformation to be minimal in its rendering
the impossible X possible, we still cannot ensure that 
we obtain a unique representation of X, so long as X is not possible.  
Hence it is understood to be not what it is unconditionally, 
but only what it is relative to a minimal transformation that applied.
The collection of 
the possible representations is sometimes described as the 
{\it extension} of X. Of course, 
one may have certain intention, under which X refers to 
some particular representations of X. They are termed {\it intension} of X for 
contrast. \\
\indent The two terms are actively differentiated 
in Intensional Logic \cite{Church51,Montague74,Carnap47}. For example, 
suppose that we have two concepts denoting collections U and V such that 
their union 
is neither U nor V. Then, although U is certainly not equal to V, if, 
for instance, we regard every concept as a designator of an element 
of the collection, then U is V if U {\small $\mapsto$} u and 
V {\small $\mapsto$} v such that u = v. For a comparison, 
\code{Grad} does not treat intension explicitly, for 
if some entity equals another in \code{Grad}, then they are always 
extensionally equal: 
if the morning star is the evening star, it cannot be because the two
terms designate the planet Venus that \code{Grad} says they are equal, but 
because they are the same. But it expresses the distinction  
passively in the sense 
that we can meta-logically observe it. To wit, 
consider an expression {\small $(\top \gtrdot \code{Space} \gtrdot 
    \code{Wide}) \wedge 
    (\code{Space} \gtrdot \code{Wide})$}. Then, depending on 
what the given domain of discourse is, the sense of  
{\small $\code{Space}$} in {\small $\top \gtrdot \code{Space} \gtrdot
    \code{Wide}$} may not be the same as that of 
{\small $\code{Space}$} in {\small $\code{Space} \gtrdot 
    \code{Wide}$}. Similarly for {\small $\code{Wide}$}. (Incidentally, 
note that {\small $\gtrdot$} is not the type/sub-type relation.)
The intensionality in the earlier mentioned conditional logics is, 
provided counterfactual statements are reasoned in line with the prescription in Appendix J,  slightly more explicit: the judgement of Y 
depends on intension of X. But in many of the ontic conditional logics 
in \cite{Horacio13}, it does not appear to be explicitly distinguished 
from extension. \\
\indent It 
could be the case that \code{Grad}, once 
extended with predicates, may be able to express intensionality in a 
natural way, \emph{e.g.} we may say {\small $\exists \code{Intension}(\code{Adjective} \gtrdot \code{Sheep})
    = \code{Ovine}$} (in some, and not necessarily all,
sub-domains of discourse). 
At any rate, how much we should care for the distinction of 
intensionality and extensionality probably owes much to personal tastes. 
We may study intensionality as an independent component to be added 
to extensional logics. We may alternatively study a logic in which extensionality 
is deeply intertwined with intensionality. It should be 
the sort of applications we have in mind that favours 
one to the other. \\
\indent Of the logics that touch upon concepts, also 
worth  mentioning are a family 
of description logics \cite{Baader10} that have influence in 
knowledge representation. 
They are a fragment of the first-order logic specialised 
in setting up knowledge bases, in reasoning about their contents and 
in manipulating them \cite{Baader10-2}. 
The domain of discourse, a knowledge base, is 
formed of two components. One called TBox stores 
knowledge that does not usually change over time: 
(1) concepts (corresponding to unary predicates in the first-order logic)
and 
(2) roles (corresponding to binary predicates), specifically. The other one, ABox, 
stores contingent knowledge of 
assertions about individuals, \emph{e.g.} Mary, an individual, is mother, 
a general concept. Given the domain of discourse, there then are 
reasoning facilities 
in description logics responsible for checking satisfiability of 
an expression as well as for judging whether one description
is a sub-/super-concept of another (here a super-concept of a concept
is not ``a concept of a concept'' in the term of \cite{Church51}). 
\\
    \indent Description logics were developed from specific applications, 
    and capture a rigid sense of the concept. 
    It should be of interest to see how 
    \code{Grad} may be specialised for applications in computer science. 
    To see if the use of {\small $\gtrdot$} as a meta-relation on 
    description logic instances 
    can lead to results that have been conventionally 
    difficult to cope with is another hopeful direction. 
\subsection{Combined Logic}      
\code{Grad} is a particular 
kind of combined logic \cite{stanford11,Gabbay96a,Caleiro05} 
combining the same logic 
over and over finitely many times. The presence of 
the extra logical connective {\small $\gtrdot$} scarcely  
diverts it from the philosophy of combined logics.  
Instead of  
regarding base logics\footnote{A base logic of a combined logic is one 
that is used to derive the combined logic.} as effectively bearing
the same significance in footing, however, this work recognised certain 
sub-ordination between base logics, as the new logical connective 
characterised. Object-attribute negation also bridges across  
the base logics. Given these, 
a finite number 
of the base logic combinations at once made more sense than 
combinations of two base logics finitely many times, 
for the latter approach 
may not be able to adequately represent the meta-base-logic 
logical connectives with the intended semantics of gradual 
classical logic. Investigation  
into this sub-set of combined logics could have merits of its own.  
\hide{which 
was contemplated by Gabbay in his fifth agenda \cite{Gabbay96a} and was 
reminded again in \cite{Arisaka13}. }
\subsection{Conclusion}     
This work presented \code{Grad} as a logic for 
attributed objects. Its mechanism should 
be easily integrated into many non-intuitionistic logics.  
Directions to future research were also suggested 
at lengths through comparisons. Considering its variations 
should be also interesting. 
For applications of gradual logics, 
program analysis/verification, databases, and artificial intelligence 
come into mind. 
\bibliographystyle{plain} 
\bibliography{references}   

\begin{thebibliography}{10}

\bibitem{Horacio13}
Horacio Arlo-Costa.
\newblock The {Logic} of {Conditionals}.
\newblock In Edward~N. Zalta, editor, {\em The {Stanford} {Encyclopedia} of
  {Philosophy}}. 2013.

\bibitem{Baader10}
Franz Baader, Diego Calvanese, Deborah McGuinness, Daniele Nardi, and Peter~F.
  Patel-Schneider, editors.
\newblock {\em {The} {Description} {Logic} {Handbook}: {Theory},
  implementation, and applications}.
\newblock Cambridge University Press, 2010.

\bibitem{Baader10-2}
Franz Baader and Werner Nutt.
\newblock Basic description logics.
\newblock In {\em {The} {Description} {Logic} {Handbook}: {Theory},
  implementation, and applications}. Cambridge University Press, 2010.

\bibitem{Caleiro05}
Carlos Caleiro, Walter Carnielli, Jo\ ao~Rasga, and Cristina Sernadas.
\newblock Fibring of logics as a universal construction.
\newblock {\em Handbook of Philosophical Logic}, 13:123--187, 2005.

\bibitem{Carnap47}
Rudolf Carnap.
\newblock {\em Meaning and Necessity}.
\newblock The {University} {Chicago} {Press}., 1947.

\bibitem{stanford11}
Walter Carnielli and Marcelo~E. Coniglio.
\newblock Combining logics.
\newblock In Edward~N. Zalta, editor, {\em The {Stanford} {Encyclopedia} of
  {Philosophy}}. 2011.

\bibitem{Chellas80}
Brian~F. Chellas.
\newblock {\em Modal {Logic}: {An} {Introduction}}.
\newblock Cambidge University Press, 1980.

\bibitem{Church51}
Alonzo Church.
\newblock A {Formulation} of the {Logic} of {Sense} and {Denotation}.
\newblock In Paul Henle, Horace~M. Kallen, and Susanne~K. Langer, editors, {\em
  Structure, {Method} and {Meaning}:{Essays} in {Honor} of {Henry} {M}.
  {Sheffer}}, pages 3--24. The {Liberal} {Arts} {Press}, 1951.

\bibitem{Gabbay96a}
Dov~M. Gabbay.
\newblock Fibred semantics and the weaving of logics: Part 1.
\newblock {\em J. Symb. Log.}, 61(4):1057--1120, 1996.

\bibitem{DBLP:journals/tcs/Girard87}
Jean-Yves Girard.
\newblock Linear logic.
\newblock {\em Theor. Comput. Sci.}, 50:1--102, 1987.

\bibitem{Gottwald09}
Siegfried Gottwald.
\newblock Many-valued {Logic}.
\newblock In {\em The {Stanford} {Encyclopedia} of {Philosophy}}. 2009.

\bibitem{Hajek10}
Petr H{\'a}jek.
\newblock Fuzzy {Logic}.
\newblock In {\em The {Stanford} {Encyclopedia} of {Philosophy}}. 2010.

\bibitem{Hendricks06}
Vincent Hendricks and John Symons.
\newblock Epistemic {Logic}.
\newblock In {\em The {Stanford} {Encyclopedia} of {Philosophy}}. 2006.

\bibitem{Kant08}
Immanuel Kant.
\newblock {\em Critique of Pure Reason}.
\newblock Penguin Classics; Revised edition, 2008.

\bibitem{Kleene52}
Stephen~C. Kleene.
\newblock {\em Introduction to META-MATHEMATICS}.
\newblock North-Holland Publishing Co., 1952.

\bibitem{Lewis01}
David~K. Lewis.
\newblock {\em Counterfactuals}.
\newblock Wiley-Blackwell, 2nd edition, 2001.

\bibitem{Marcos05}
Jo{\~a}o Marcos.
\newblock {\em Logics of {Formal} {Inconsistency}}.
\newblock PhD thesis, Universidade T\'ecnica De Lisboa, 2005.

\bibitem{Montague74}
Richard Montague.
\newblock On the {Nature} of {Certain} {Philosophical} {Entities}.
\newblock In Richmond~H. Thomason, editor, {\em Formal {Philosophy}: {Selected}
  {Papers} of {Richard} {Montague}}, pages 148--188. Yale {University} {Press},
  1974.

\bibitem{Plumpton31}
Frank~Plumpton Ramsey.
\newblock General {Propositions} and {Causality}.
\newblock In {\em The {Foundations} of {Mathematics} and other {Logical}
  {Essays}}, pages 237--255. Kegan {Paul}, {Trench} \& {Trubner}, 1931.

\bibitem{wikiBooleanAlgebra}
Wikipedia.
\newblock Boolean algebra.
\newblock http://en.wikipedia.org/wiki/Boolean\_algebra.

\end{thebibliography}
\newpage
\section*{Appendix A: Proof of Lemma \ref{linking_principle}}   
  First apply {\small $\gtrdot$} reductions 2 and 3 on 
      {\small $F_1 \gtrdot F_2$} into a formula 
      in which the only occurrences of the chains are 
      {\small $f_0 \gtrdot 
      F_2$}, {\small $f_1 \gtrdot F_2$}, \dots, {\small $f_{k} 
      \gtrdot F_2$} for some {\small $k \in \mathbb{N}$} and 
      some {\small $f_0, f_1, \dots, f_k \in \mathfrak{U} 
      \cup \mathcal{S}$}. Then apply {\small $\gtrdot$} reductions 
      4 and 5 to each of those chains into a formula 
      in which the only occurrences of the chains are: 
      {\small $f_0 \gtrdot g_{0}, f_0 \gtrdot g_{1}, \dots, 
      f_0 \gtrdot g_{j}$},
      {\small $f_1 \gtrdot g_{0}$}, \dots, 
      {\small $f_1 \gtrdot g_{j}$}, \dots, 
      {\small $f_k \gtrdot g_{0}$}, \dots, {\small $f_k \gtrdot g_j$} 
      for some {\small $j \in \mathbb{N}$} and 
      some {\small $g_0, g_1, \dots, g_j \in \mathfrak{U}$}. 
      To each such chain, apply 
      {\small $\gtrdot$} reduction 1 as long as it is applicable. 
      This process cannot continue infinitely since any formula 
      is finitely constructed and since, under the premise, we 
      can apply induction
      on the number of elements of {\small $\mathcal{S}$} 
      occurring in {\small $g_x$}, {\small $0 \le x 
      \le j$}. 
      The straightforward inductive proof is left to readers. The 
      result is a formula in unit chain expansion. \\
      \section*{Appendix B: Proof of Lemma \ref{reduction_result}} 
  By induction on maximal number of {\small $\neg$} nestings 
    and a sub-induction on formula size. 
        We quote Lemma \ref{normalisation_without_negation} for 
    base cases. For inductive cases, assume that the current 
    lemma holds true for all the formulas with {\small $\negMax(F_0)$} 
    of up to {\small $k$}. Then we conclude by showing that 
    it still holds true 
    for all the formulas with {\small $\negMax(F_0)$} of {\small $k+1$}. 
    Now, because any formula is finitely constructed,  
    there exist sub-formulas in which occur no {\small $\neg$}. 
    By Lemma \ref{normalisation_without_negation}, those sub-formulas 
    have a reduction into a formula in unit chain expansion. Hence 
    it suffices to show that those formulas 
    {\small $\neg F'$} with {\small $F'$} already in unit chain 
    expansion reduce into a formula in unit chain expansion, upon which
    inductive hypothesis applies for a conclusion. 
    Consider what {\small $F'$} is: 
    \begin{enumerate}
      \item {\small $s$}: then apply {\small $\neg$} reduction 1 
	on {\small $\neg F'$} to remove the {\small $\neg$}
	occurrence. 
      \item {\small $F_a \wedge F_b$}: apply {\small $\neg$} 
	reduction 2. 
	Then apply (sub-)induction hypothesis on 
	{\small $\neg F_a$} and {\small $\neg F_b$}.  
      \item {\small $F_a \vee F_b$}: apply {\small $\neg$} 
	reduction 3. Then apply (sub-)induction hypothesis on 
	{\small $\neg F_a$} and {\small $\neg F_b$}. 
      \item {\small $s \gtrdot F \in \mathfrak{U}$}: apply 
	{\small $\neg$} reduction 4. Then apply 
	(sub-)induction hypothesis on {\small $\neg F$}.
    \end{enumerate} 
    \section*{Appendix C: Proof of Lemma \ref{unit_chain_excluded_middle}} 
 (Note again that we are assuming well-formed 
 formulas only.) 
 For the first one,
 {\small $[\intFrame \models_D s_0 \gtrdot  
     s_1 \gtrdot 
 \dots \gtrdot s_{k}] = 1$} implies that
 {\small $\mathsf{I}(\epsilon, s_0)\! =\! 
 \mathsf{I}(s_0, s_1) \!=\! \dots \!=\! 
 \mathsf{I}(s_0.s_1.\dots.s_{k - 1}, s_{k})\! =\! 1$}.
 So we have; 
 {\small $\mathsf{I}(\epsilon, s_0^c) = 
 \mathsf{I}(s_0, s_1^c) = \dots = 
 \mathsf{I}(s_0.s_1\dots.s_{k - 1}, s_{k}^c) = 0$} by the 
 definition of {\small $\mathsf{I}$}. 
 Meanwhile, 
 {\small $\recurseReduce(s_0 \gtrdot s_1 \gtrdot \cdots \gtrdot 
 s_k) = s_0^c \vee (s_0 \gtrdot ((s_1^c \vee (s_1 \gtrdot \cdots)))) 
 = s_0^c \vee (s_0 \gtrdot s_1^c) \vee (s \gtrdot s_1 \gtrdot 
 s_2^c) \vee \cdots \vee 
 (s \gtrdot s_1 \gtrdot \cdots \gtrdot s_{k-1} \gtrdot s_k^c)$}. 
 Therefore {\small $[\intFrame \models_D
 \recurseReduce(s_0 \gtrdot s_1 \gtrdot \cdots \gtrdot s_k)]  
 = 0 \not= 1$} for the given valuation frame. \\
 \indent For the second obligation, 
 {\small $[\intFrame \models_D
 s_0 \gtrdot s_1 \gtrdot \dots \gtrdot s_{k}] = 0$} 
 implies that   
 {\small $[\mathsf{I}(\epsilon, s_0) 
 = 0] \vee^{\dagger} [\mathsf{I}(s_0, s_1) = 0] 
 \vee^{\dagger} \dots \vee^{\dagger} [\mathsf{I}(s_0.s_1.\dots.
 s_{k -1}, s_{k}) = 0]$}. Again 
 by the definition of {\small $\mathsf{I}$}, 
 we have the required result. That these two events 
 are mutually exclusive is trivial. 
 \section*{Appendix D: Proof of Proposition \ref{associativity_commutativity_distributivity}} 
Let us generate a set of 
  expressions finitely constructed from the following grammar;\\
  \indent {\small $X := [\intFrame \models_D f] \ | \ X \wedge^{\dagger} X 
  \ | \ X \vee^{\dagger} X$} where {\small $f \in \mathfrak{U} \cup 
  \mathcal{S}$}. \\
  Then first of all it is straightforward to 
  show that {\small $[\mathfrak{M} \models_D F_i] = X_i$} 
  for each {\small $i \in \{1,2,3\}$} for some {\small $X_1, X_2, X_3$} 
  that the above grammar recognises. By Lemma \ref{unit_chain_excluded_middle} each atomic expression ({\small $[\mathfrak{M} \models_D f_x]$} for 
  some {\small $f_x \in \mathfrak{U} \cup \mathcal{S}$}) 
  is assigned one and only one value {\small $v \in \{0,1\}$} (again note that we are considering well-formed 
  formulas only). Then 
  since {\small $1 \vee^{\dagger} 1 = 1 \vee^{\dagger} 0 = 
  0 \vee^{\dagger} 1 = 1$}, 
  {\small $0 \wedge^{\dagger} 0 = 0 \wedge^{\dagger} 1 = 
  1 \wedge^{\dagger} 0 = 0$}, and 
  {\small $1 \wedge^{\dagger} 1 = 1$} by definition given at 
  the beginning 
  of this section, 
  it is also the case that {\small $[\mathfrak{M} \models_D F_i]$} 
  is assigned one and only one value {\small $v_i \in  \{0,1\}$} 
  for each {\small $i \in \{1,2,3\}$}. Then the proof for the 
  current proposition is straightforward.  
  \section*{Appendix E: Proof of Lemma \ref{unit_double_negation}}   
{\small $\recurseReduce(\recurseReduce(F)) =  
  \recurseReduce(s^c_0 \vee (s_0 \gtrdot s^c_1) \vee 
  \cdots \vee 
  (s_0 \gtrdot s_1 \gtrdot \cdots \gtrdot s_{k-1} 
  \gtrdot s^c_{k})) = 
   s_0 \wedge (s^c_0 \vee (s_0 \gtrdot s_1)) 
  \wedge (s_0^c \vee (s_0 \gtrdot s_1^c) 
  \vee (s_0 \gtrdot s_1 \gtrdot s_2)) 
  \wedge \cdots \wedge 
  (s^c_0 \vee (s_0 \gtrdot s_1^c) \vee \cdots 
  \vee (s_0 \gtrdot s_1 \gtrdot \cdots \gtrdot s_{k-2} \gtrdot 
  s^c_{k-1}) 
  \vee (s_0 \gtrdot s_1 \gtrdot \cdots \gtrdot s_{k}))$}.   
  Here, assume that the right hand side of the equation
  which is in conjunctive normal form is ordered, 
  the number of terms, from left to right, strictly increasing 
  from 1 to {\small $k + 1$}. Then as the result of a transformation
  of the conjunctive 
  normal form into disjunctive normal form we will 
  have 1 (the choice from the first conjunctive clause which contains 
  only one term {\small $s_0$}) {\small $\times$} 
  2 (a choice from the second conjunctive clause with 
  2 terms {\small $s_0^c$} and {\small $s_0 \gtrdot s_1$}) 
  {\small $\times$} \ldots {\small $\times$} (k $+$ 1) clauses. But  
  almost all the clauses in 
  {\small $[\intFrame \models_D (\text{the disjunctive
  normal form})]$}
  will be assigned 0 (trivial; the proof left to readers) so that we gain
  {\small $[\intFrame \models_D (\text{the disjunctive normal form})] 
  = [\intFrame \models_D s_0] \wedge^{\dagger} [\intFrame \models_D
  s_0 \gtrdot s_1] \wedge^{\dagger} \cdots 
  \wedge^{\dagger} [\intFrame \models_D s_0 \gtrdot s_1 
  \gtrdot \cdots \gtrdot s_k] =
  [\intFrame \models_D s_0 \gtrdot s_1 
  \gtrdot \cdots \gtrdot s_k]$}. \\
\section*{Appendix F: Proof of Proposition \ref{excluded_middle}} 
Firstly for {\small $1 = [\intFrame \models_D F \vee \recurseReduce(F)]$}. 
  By Proposition \ref{associativity_commutativity_distributivity}, 
  {\small $F$} has a disjunctive normal form: 
  {\small $F = \bigvee_{i = 0}^{k} \bigwedge_{j=0}^{h_i}   
  f_{ij}$} for some {\small $i, j, k \in \mathbb{N}$}, 
  some {\small $h_0, \cdots, h_k \in \mathbb{N}$} 
  and some {\small $f_{00}, \cdots, f_{kh_k} \in 
  \mathfrak{U} \cup \mathcal{S}$}.  
  Then {\small $\recurseReduce(F)  
  = \bigwedge_{i=0}^k \bigvee_{j=0}^{h_i} \recurseReduce(f_{ij})$}, 
  which, if transformed into a disjunctive normal form, 
  will have {\small $(h_0 + 1)$} [a choice from 
  {\small $\recurseReduce(f_{00}), \recurseReduce(f_{01}), \dots,\\
  \recurseReduce(f_{0h_0})$}] {\small $\times$} 
  {\small $(h_1 + 1)$} [a choice from 
  {\small $\recurseReduce(f_{10}), \recurseReduce(f_{11}), \dots,\\
  \recurseReduce(f_{1h_1})$}] {\small $\times \dots \times$} 
  {\small $(h_k + 1)$} clauses. Now if 
  {\small $[\intFrame \models_D F] = 1$}, then we already have the required 
  result. Therefore suppose that {\small $[\intFrame 
      \models_D F] = 0$}. 
  Then it holds that {\small $\forall i \in \{0, \dots, k\}. 
  \exists j \in \{0, \dots, h_i\}.([\intFrame \models_D f_{ij}] = 0)$}. But 
  by Lemma \ref{unit_chain_excluded_middle}, this is equivalent to 
  saying that {\small $\forall i \in \{0, \dots, k\}. 
  \exists j \in \{0, \dots, h_i\}.([\intFrame \models_D \recurseReduce(f_{ij})] 
  = 1)$}. But then there exists a clause in disjunctive normal form 
  of {\small $[\intFrame \models_D \recurseReduce(F)]$} which is assigned 1. 
  Dually for {\small $0 = [\intFrame \models_D F \wedge \recurseReduce(F)]$}. 
  \\
  \indent For {\small $[\intFrame \models_D F] = 
  [\intFrame \models_D \recurseReduce(\recurseReduce(F))]$},  
  by Proposition \ref{associativity_commutativity_distributivity}, 
  {\small $F$} has a disjunctive normal form: 
  {\small $F = \bigvee_{i = 0}^k \bigwedge_{j=0}^{h_i} f_{ij}$} 
  for some {\small $i, j, k \in \mathbb{N}$}, 
  some {\small $h_0, \dots, h_k \in \mathbb{N}$} and 
  some {\small $f_{00}, \dots, f_{kh_k} \in \mathfrak{U} \cup 
  \mathcal{S}$}. Then {\small $\recurseReduce(\recurseReduce(F)) 
  = \bigvee_{i = 0}^{k} \bigwedge_{j=0}^{h_i} \recurseReduce(
  \recurseReduce(f_{ij}))$}. But by Lemma \ref{unit_double_negation} 
  {\small $[\intFrame \models_D \recurseReduce(\recurseReduce(f_{ij}))] = 
      [\intFrame \models_D f_{ij}]$} for each appropriate {\small $i$} and 
  {\small $j$}. Straightforward. \\
  \section*{Appendix G: Proof of Lemma \ref{bisimulation}} 
By induction on the number of reduction steps and a sub-induction 
on formula size, 
   we first establish that {\small $
       \mathcal{F}(F_1) = \mathcal{F}(F_2)$} (by bisimulation). 
   Into one way to show that to each 
   reduction on {\small $F'$} corresponds 
   reduction(s) on {\small $F$} is straightforward, 
   for we can choose to reduce {\small $F$} into {\small $F'$}, 
   thereafter we synchronize both of the reductions. Into the 
   other way to show that to each 
   reduction on {\small $F$} corresponds 
   reduction(s) on {\small $F'$}, we consider each case: 
   \begin{enumerate}
     \item The first pair.  
       \begin{enumerate} 
	 \item 
       If a reduction takes place on a sub-formula which 
       neither is a sub-formula of the shown sub-formula 
       nor has as its sub-formula the shown sub-formula, 
       then we reduce the same sub-formula in {\small $F'$}.   
       Induction hypothesis (note that the number of 
reduction steps is that of {\small $F$} into this 
direction).  
     \item If it takes place on a sub-formula 
       of {\small $F_a$} or {\small $F_b$} 
       then we reduce the same sub-formula of 
       {\small $F_a$} or {\small $F_b$} in {\small $F'$}. Induction 
hypothesis. 
     \item If it takes place on a sub-formula 
       of {\small $F_c$} then we reduce the same sub-formula 
       of both occurrences of {\small $F_c$} in {\small $F'$}.  
Induction hypothesis. 
     \item If {\small $\gtrdot$} reduction 2 takes place on 
       {\small $F$} such that we have; 
       {\small $F[(F_a \wedge F_b) \gtrdot F_c] \leadsto 
       F_x[(F_a \gtrdot F_c) \wedge (F_b \gtrdot F_c)]$} where 
       {\small $F$} and {\small $F_x$} differ only by 
       the shown sub-formulas,\footnote{This note `where \dots' is assumed in the 
       remaining.} then do nothing on {\small $F'$}. And {\small $F_x = 
       F'$}. Vacuous thereafter. 
     \item If {\small $\gtrdot$} reduction 2 takes place 
       on {\small $F$} such that we have; 
       {\small $F[(F_d \wedge F_e) \gtrdot F_c] \leadsto 
       F_x[(F_d \gtrdot F_c) \wedge (F_e \gtrdot F_c)]$} 
       where {\small $F_d \not= F_a$} and {\small $F_d \not = F_b$}, 
       then without loss of generality assume that 
       {\small $F_d \wedge F_{\beta} = F_a$} 
       and that {\small $F_{\beta} \wedge F_b = F_e$}. 
       Then we apply {\small $\gtrdot$} reduction 2 
       on the {\small $(F_d \wedge F_{\beta}) \gtrdot F_c$} in 
       {\small $F'$} so that we have; 
       {\small $F'[((F_d \wedge F_{\beta}) \gtrdot F_c) \wedge 
       (F_b \gtrdot F_c)] \leadsto 
       F''[(F_d \gtrdot F_c) \wedge (F_{\beta} \gtrdot F_c) 
       \wedge (F_b \gtrdot F_c)]$}. 
       Since {\small $(F_x[(F_d \gtrdot F_c) \wedge (F_e \gtrdot F_c)] 
       =) F_x[(F_d \gtrdot F_c) \wedge ((F_{\beta} \wedge F_b) \gtrdot 
       F_c)] = F_x'[(F_{\beta} \wedge F_b) \gtrdot F_c]$} and 
       {\small $F''[(F_d \gtrdot F_c) \wedge (F_{\beta} \gtrdot F_c) 
       \wedge (F_b \gtrdot F_c)] = F'''[(F_{\beta} \gtrdot F_c) 
       \wedge (F_b \gtrdot F_c)]$} such that 
       {\small $F'''$} and {\small $F_x'$} differ only 
       by the shown sub-formulas, we repeat the rest of simulation 
       on 
       {\small $F'_x$} and {\small $F'''$}. Induction hypothesis. 
     \item If a reduction takes place on a sub-formula {\small $F_p$} of 
       {\small $F$} in which the shown sub-formula of 
       {\small $F$} occurs as a strict sub-formula 
       ({\small $F[(F_a \wedge F_b) \gtrdot F_c] 
       = F[F_p[(F_a \wedge F_b) \gtrdot F_c]]$}), then 
       we have {\small $F[F_p[(F_a \wedge F_b) \gtrdot F_c]] 
       \leadsto F_x[F_q[(F_a \wedge F_b) \gtrdot F_c]]$}. 
       But we have 
       {\small $F' = F'[F_p'[(F_a \gtrdot F_c) \wedge (F_b \gtrdot 
       F_c)]]$}. Therefore we apply the same reduction on 
       {\small $F_p'$} to gain; 
       {\small $F'[F_p'[(F_a \gtrdot F_c) \wedge (F_b \gtrdot 
       F_c)]] \leadsto F'_x[F_{p'}'[(F_a \gtrdot F_c) \wedge (F_b 
       \gtrdot F_c)]]$}. Induction hypothesis. 
   \end{enumerate} 
 \item The second, the third and the fourth pairs: Similar. 
 \item The fifth pair: 
   \begin{enumerate}
     \item If a reduction takes place on a sub-formula 
       which neither is a sub-formula of the shown 
       sub-formula nor has as its sub-formula the shown sub-formula, 
       then we reduce the same sub-formula in {\small $F'$}.  
Induction hypothesis. 
     \item If it takes place on a sub-formula of {\small $F_a$}, 
       {\small $F_b$} or {\small $F_c$}, then 
       we reduce the same sub-formula of all the occurrences
       of the shown {\small $F_a$}, {\small $F_b$} or {\small $F_c$} 
       in {\small $F'$}. Induction hypothesis. 
     \item If {\small $\gtrdot$} reduction 4 takes place on 
       {\small $F$} such that we have; 
       {\small $F[(F_a \gtrdot F_b) \gtrdot F_c]
       \leadsto 
       F_x[(F_a \gtrdot F_c) \wedge ((F_a \gtrdot F_b) \vee 
       (F_a \gtrdot F_b \gtrdot F_c))]$}, then do nothing on 
       {\small $F'$}. And {\small $F_x = F'$}. Vacuous thereafter. 
     \item If a reduction takes place on a sub-formula 
       {\small $F_p$} of {\small $F$} in which the shown 
       sub-formula of {\small $F$} occurs 
       as a strict sub-formula, then similar to the case 1) f).   
   \end{enumerate}
   \end{enumerate} 
   By the result of the above bisimulation, we now have  
   {\small $\mathcal{F}(F) = \mathcal{F}(F')$}. However,
   without {\small $\neg$} occurrences in {\small $F$} it takes 
   only those 5 {\small $\gtrdot$} reductions to 
   derive a formula in unit chain expansion; hence we in fact have
   {\small $\mathcal{F}(F) = \mathcal{F}(F_x)$} for some 
   formula {\small $F_x$} in unit chain expansion. But 
   then by Theorem \ref{theorem_1}, there could be 
   only one value out of {\small $\{0,1\}$} assigned to {\small $[\intFrame \models_D F_x]$} if 
   {\small $F_x$} is well-formed; otherwise, 
   {\small $\code{illFormed}$} is assigned. 
   \\
   \section*{Appendix H: Proof of Lemma \ref{other_bisimulation}}
By simultaneous induction on reduction steps and by 
   a sub-induction on formula size. One way is trivial. Into the direction to 
   showing that to every reduction on {\small $F$} corresponds reduction(s)
   on {\small $F'$}, we consider each case. For the first case;
       \begin{enumerate}
	 \item If a reduction takes place on a sub-formula 
	   which neither is a sub-formula of the shown sub-formula 
	   nor has as its sub-formula the shown sub-formula, 
	   then we reduce the same sub-formula in {\small $F'$}.   
   Induction hypothesis. 
	 \item If it takes place on a sub-formula of {\small $F_a$} 
	   or {\small $F_b$} then we reduce the same sub-formula 
	   of {\small $F_a$} or {\small $F_b$} in {\small $F'$}.  
Induction hypothesis. 
	 \item If {\small $\neg$} reduction 2 takes place 
	   on {\small $F$} such that we have;  
	   {\small $F[\neg (F_a \wedge F_b)] \leadsto 
	   F_x[\neg F_a \vee \neg F_b]$}, then do nothing 
	   on {\small $F'$}. And {\small $F_x = F'$}. Vacuous thereafter.
	 \item If {\small $\neg$} reduction 2 takes place 
	   on {\small $F$} such that we have; 
	   {\small $F[\neg (F_d \wedge F_e)] 
	   \leadsto F_x[\neg F_d \vee \neg F_e]$} where 
	   {\small $F_d \not= F_a$} and {\small $F_d \not= F_b$}, 
	   then without loss of generality assume that 
	   {\small $F_d \wedge F_{\beta} = F_a$} 
	   and that {\small $F_{\beta} \wedge F_b = F_e$}.  
	   Then we apply {\small $\neg$} reduction 2 on the 
	   {\small $\neg (F_d \wedge F_{\beta})$} in 
	   {\small $F'$} so that we have; 
	   {\small $F'[\neg (F_d \wedge F_{\beta}) \vee 
	   \neg F_b] \leadsto F''[\neg F_d \vee \neg F_{\beta} 
	   \vee \neg F_b]$}. Since 
	   {\small $(F_x[\neg F_d \vee \neg F_e] = ) 
	   F_x[\neg F_d \vee \neg (F_{\beta} \wedge F_b)] 
	   = F'_x[\neg (F_{\beta} \wedge F_b]$} and 
	   {\small $F''[\neg F_d \vee \neg F_{\beta} 
	   \vee \neg F_b] = F'''[\neg F_{\beta} \vee \neg F_b]$} 
	   such that {\small $F'''$} and {\small $F'_x$} differ only 
	   by the shown sub-formulas, we repeat the rest of 
	   simulation on {\small $F'_x$} and {\small $F'''$}.  
Induction hypothesis. 
	 \item If a reduction takes place on a sub-formula 
	   {\small $F_p$} of {\small $F$} in which the shown 
	   sub-formula of {\small $F$} occurs 
	   as a strict sub-formula, then similar to 
	   the 1) f) sub-case in 
	   Lemma \ref{bisimulation}. 
       \end{enumerate}  
       The second case is similar. For the third case; 
       \begin{enumerate} 
\item If no reduction is applicable,  
                       then vacuously 
                       {\small $[\intFrame \models_D
                           F] = [\intFrame \models_D
                           F']$}. 
           \item If a reduction takes place on a sub-formula 
               which neither is a sub-formula of the shown 
               sub-formula nor has as its sub-formula 
               the shown sub-formula, then we reduce 
               the same sub-formula in {\small $F'$}. Induction 
               hypothesis.  
           \item If a reduction takes place on a sub-formula 
               {\small $F_p$} of {\small $F$} in which 
               the shown sub-formula of {\small $F$} occurs 
               as a strict sub-formula, then; 
               \begin{enumerate}   
                    \item If the applied reduction is 
                     $\neg$ reduction 2 or 4, then straightforward. 
                  \item If the applied reduction is  
                      $\neg$ reduction 3 such that 
                      {\small $(F =  F_a[\neg (F_x \vee s \vee s \vee F_y)])
                      \leadsto (F_b[\neg F_x \wedge \neg s \wedge \neg s \wedge 
                      \neg F_y] 
                      = F_c[\neg s \wedge \neg s]) \leadsto 
                      F_d[s^c \wedge s^c]$} for some 
                  {\small $F_x$} and {\small $F_y$} (the last 
                  transformation does not cost generality due to simultaneous 
                  induction),
                  then we reduce 
                  {\small $F'$} as follows: 
                  {\small $(F' = F_a'[\neg (F_x \vee s \vee F_y)]) 
                      \leadsto (F_b'[\neg F_x \wedge \neg s \wedge \neg 
                      F_y] = F_c'[\neg s]) \leadsto F_d'[s^c]$}. 
                  Induction hypothesis. Any other cases 
                  are straightforward.  
              \item If the applied reduction is $\gtrdot$ reduction 
                  1-4, then straightforward. 
            \end{enumerate}
       \end{enumerate} 
       Similarly for the remaing ones. 
       \section*{Appendix I: Proof of Theorem \ref{theorem_normalisation}}
By induction on maximal number of {\small $\neg$} nestings 
    and a sub-induction on formula size. We quote Lemma 
    \ref{reduction_without_negation} for base cases. 
    For inductive cases, assume that the current theorem 
    holds true for all the formulas with {\small $\negMax(F_0)$} 
    of up to {\small $k$}. Then we conclude by showing 
    that it still holds true for all the formulas with 
    {\small $\negMax(F_0)$} of {\small $k+1$}. First we note that 
    there applies no {\small $\neg$} reductions on {\small $\neg F_x$} 
    if {\small $F_x$} is a chain whose head is not 
    an element of {\small $\mathcal{S}$}. But this is 
    straightforward from the descriptions of the reduction rules. \\
    \indent On this observation we show that 
    if we have a sub-formula {\small $\neg F_x$} such that 
    no {\small $\neg$} occurs in {\small $F_x$}, then 
    {\small $F_x$} can be reduced into a formula in unit chain 
    expansion with no loss of generality, prior to 
    the reduction of the outermost {\small $\neg$}. Then 
    we have the desired result by induction hypothesis 
    and the results in the previous sub-section. But suppose otherwise.  
    Let us denote by {\small $\mathcal{F}$} the set of formulas in unit chain 
    expansion that {\small $\neg F_x'$} reduces into where
    {\small $F_x'$} is a unit chain expansion of {\small $F_x$}.
    Now suppose there exists {\small $F_y$} in unit chain 
    expansion that 
    {\small $\neg F_x$} can reduce into if the outermost 
    {\small $\neg$} reduction applies before {\small $F_x$} 
    has 
    reduced into a formula in unit chain expansion such as to 
    satisfy that {\small $[\intFrame \models_D F_y] \not= [\intFrame \models_D F_{\beta}]$} 
    for some {\small $F_{\beta} \in \mathcal{F}$}. 
    We here have; \\
    {\small $\neg F_x \leadsto^*_{\{\gtrdot\!\! \text{ 
    reductions only}\}} \neg F_z \leadsto^*_{\{\gtrdot\!\! \text{ 
    reductions only}\}} 
    \neg F_x' \leadsto^+_{\{\neg \text{ reductions
    only}\}}\!\! F_{\beta}$} and 
    {\small $\neg F_x \leadsto^*_{\{\gtrdot \text{ 
    reductions only}\}} \neg F_z \leadsto_{\neg \text{ reduction}}
    F_z' \leadsto^* F_y$} where {\small $\neg^{\dagger} 
    \exists F_{zz}.F_z' = \neg F_{zz}$}. \\
    Hence for our supposition to hold, it must satisfy that 
    there exists no bisimulation between {\small $F'_z$} 
    and {\small $\neg F_z$}. But because it is trivially provable 
    that to each reduction on {\small $F'_z$} corresponds 
    reduction(s) on {\small $\neg F_z$} (, for we can choose 
    to apply the {\small $\neg$} reduction on {\small $\neg F_z$} 
    to gain {\small $F'_z$},) it must in fact satisfy that 
    not to each reduction on {\small $\neg F_z$} corresponds 
    reduction(s) on {\small $F'_z$}. Consider what reduction 
    applies on a sub-formula of {\small $\neg F_z$}: 
    \begin{enumerate}
      \item any {\small $\neg$} reduction: 
	Then the reduction generates {\small $F'_z$}. 
	A contradiction to supposition has been drawn. 
      \item {\small $\gtrdot$} reduction 1:   
	   Consider how {\small $F_z$} looks like:  
	   \begin{enumerate}
	     \item {\small $F_z = F_1[(F_u \gtrdot F_v) 
	       \gtrdot F_w] \wedge F_2$}:  But then 
	       the same reduction can take place 
	       on {\small $F_z' = 
	       \neg F_1[(F_u \gtrdot F_v) \gtrdot F_w] 
	       \vee \neg F_2$}. Contradiction.  
	     \item {\small $F_z = F_1 \wedge F_2[(F_u \gtrdot F_v) 
	       \gtrdot F_w]$}: Similar. 
	     \item {\small $F_z = F_1[(F_u \gtrdot F_v) 
	       \gtrdot F_w] \vee F_2$}: Similar. 
	     \item {\small $F_z = F_1 \vee F_2[(F_u \gtrdot F_v) 
	       \gtrdot F_w]$}: Similar.  
	     \item {\small $F_z = (F_u \gtrdot F_v) \gtrdot F_w$}: 
	       This case is impossible due to the observation given 
           earlier in the current proof. 
	     \item {\small $F_z = (F_1[(F_u \gtrdot F_v) \gtrdot F_w] 
	       \gtrdot F_2) \gtrdot F_3$}: Similar.  
	     \item The rest: all similar. 
	   \end{enumerate}
      \item {\small $\gtrdot$} reduction 2: Similar. 
      \item {\small $\gtrdot$} reduction 3: Similar. 
      \item {\small $\gtrdot$} reduction 4: Consider 
	how {\small $F_z$} looks like: 
	\begin{enumerate}
	  \item {\small $F_z = s \gtrdot (F_1 \wedge F_2)$}:  
            Then {\small $\neg F_z \leadsto  
	    \neg ((s \gtrdot F_1) \wedge (s \gtrdot F_2))$}. 
	    But by Lemma \ref{other_bisimulation}, it does not cost generality if 
	    we reduce the {\small $\neg$} to have; 
	    {\small $\neg ((s \gtrdot F_1) \wedge (s \gtrdot F_2)) 
	    \leadsto \neg (s \gtrdot F_1) \vee \neg (s \gtrdot F_2)$}. 
            Meanwhile 
	    {\small $F'_z = s^c \vee (s \gtrdot \neg (F_1 \wedge F_2))$}.
	    By Lemma \ref{other_bisimulation}, it does not cost generality if 
	    we have {\small $F''_z = s^c \vee (s \gtrdot (\neg F_1 
	    \vee \neg F_2))$} instead of {\small $F'_z$}.  
	    But it also does not cost generality (by 
	    Lemma \ref{bisimulation}) 
	    if we have {\small $F'''_z = s^c \vee 
	    (s \gtrdot \neg F_1) \vee (s \gtrdot \neg F_2)$} instead 
	    of {\small $F''_z$}. But by Lemma \ref{other_bisimulation}, 
	    it again does not cost generality 
	    if we have {\small $F''''_z = s^c \vee (s \gtrdot \neg F_1) 
	    \vee s^c \vee (s \gtrdot \neg F_2)$} instead.
	    Therefore we can conduct 
	    bisimulation between {\small $\neg (s \gtrdot F_1)$} 
	    and {\small $s^c \vee (s \gtrdot \neg F_1)$} and between 
	    {\small $\neg (s \gtrdot F_2)$} 
	    and {\small $s^c \vee (s \gtrdot \neg F_2)$}. 
	    Since each of {\small $\neg (s \gtrdot F_1)$} 
	    and {\small $\neg (s \gtrdot F_2)$} has a strictly 
	    smaller formula size than 
	    {\small $\neg (s \gtrdot (F_1 \wedge F_2))$}, (sub-)induction
	    hypothesis. Contradiction.  
	  \item The rest: Trivial. 
	\end{enumerate}
      \item {\small $\gtrdot$} reduction 5: Similar. 
    \end{enumerate}
    \section*{Appendix J: An essay on reasoning about counterfactuals}  
    Conditional logics were motivated by counterfactuals. What follows 
    is but a personal viewpoint on the process of reasoning 
    about counterfactuals. Earlier ideas in the line of 
    Stalnaker's and others' \cite{Horacio13} helped sharpen 
    this view. 
    An essential purpose of reasoning about counterfactuals 
is, to 
the author at least, in conducting a partial examination on the faculty of 
our imagination. 
A `flying emu' which is considered to be non-existing can 
be, despite all the contradictions that the term causes against what 
we find within the knowledge, accommodated in our imagination. 
There what the knowledge says is the state of being flying and 
what it says is something that is an emu are refined into combinable forms 
so that a flying emu comes to existence within the parallel consciousness. 
But because 
it does not exist in the knowledge, taken two volunteers who 
are for simplicity supposed sharing the same knowledge, even if 
the flying emu in imagination of one of them does not coincide in 
features with 
that in the other, they cannot be said to be unjustifiable as 
a proper representation for  
the mismatch found between them, since no definition of  a flying emu 
is in any case found in the knowledge. 
Taken countably many volunteers, 
it comes of no surprise if the number of 
representations of a flying emu is also countably many, each one of which 
is justified as a proper in each respective imagination space.  
Therefore, 
for counterfactuals in particular out of other forms that imagination 
enables us, if we have ``If there were a flying emu, then Y would be the case," one plausible way of obtaining 
the truth value for this expression is as stated in the following pseudocode: 
\begin{description}
    \item[(Pre-condition)]{\ }\\ Some domain of discourse D is given. 
        For intuition, assume that D represents the mind of an individual. 
        D is assumed to be a logical space. 
        Knowledge is what holds in D in which, 
        like in gradual classical logic, nothing can designate a unique and 
        indivisible object (\emph{Cf.} Introduction and Postulate 1): 
        any one of them may be precise enough but never unique. 
        We suppose that `emu' and `the state of being flying' are 
        in the knowledge.  
    \item[L1] If either `flying emu' or Y holds in knowledge, then return false.  
    \item[L2] Duplicate knowledge. Apply a function F: knowledge $\rightarrow$ 
        knowledge such that F(duplicated space) $\subseteq$ knowledge. 
        The F(duplicated space) is what we here 
        call imagination. 
    \item[L3] Let us mean by refinement of an element in imagination
        its enlargement by means of any element(s) 
        that are presently found in imagination acting upon it in 
        the manner lawful to D. In 
        this term of refinement, 
        keep refining elements of imagination 
        insofar as 
        such refinement is strictly necessary to generate 
        a flying emu in the imagination. 
        Call the state of the updated imagination 
        Im if it is not inconsistent. Here, by such a refinement 
        being strictly necessary, we mean that (1) 
        the flying emu in Im ceases to exist if
        the last change that was taken to derive Im is undone, and that 
        (2) any changes made to the elements of (Im$\backslash \{\text{flying emu}\})\backslash$(knowledge) cannot be any smaller for 
        the particular flying emu to not cease to exist. 
    \item[L4$_1$] If there is no such Im, that is, 
        if no imagination space in which a flying emu exists
        derives from the duplicated knowledge 
        following the prescribed 
        alteration process such that it be contained within the 
        boundary of D, 
        then return true. 
    \item[L4$_2$] For each such Im, do: 
    \item[L4$_3$] If Y is not the case in Im, then return false.  
    \item[L4$_4$] End of For loop
    \item[L5] Return true. 
\end{description}   
A couple of relevant points are: 
(1) If a counterfactual {\small $a > b$} is true as judged by 
the above pseudo-code, then it is true by the sense delimited by 
D. (2) A counterfactual is an impossible case: if it were possible, 
it would not be a counterfactual. Hence if by a possible world we mean to refer to a world which may just as feasibly 
exist as our own world, there is no possible 
world that makes the antecedent of the counterfactual true, 
for if in some possible alternative world the antecedent were true, 
the statement 
would not be a counterfactual to the reasoning body, which goes against 
the supposition that it is a counterfactual. The antecedent is always 
false in every alternative world that a reasoning body could 
consider possible. (3) Under the stated truth judgement, we have that {\small $\neg (a > b)$} is true if and only if 
it is not the case that {\small $a > b$} is true. 
\end{document}